\documentclass{article}

\usepackage[main,final,nonatbib]{neurips_2026}
\usepackage[numbers]{natbib}

\usepackage{xcolor}         %
\usepackage{microtype}
\usepackage[pagebackref=true,backref=true]{hyperref}
\renewcommand*{\backrefalt}[4]{%
    \ifcase #1 \footnotesize{(Not cited.)}%
    \or        \footnotesize{(Cited on page~#2.)}%
    \else      \footnotesize{(Cited on pages~#2.)}%
    \fi}
\usepackage{url}
\usepackage{booktabs}
\usepackage{lineno}

\definecolor{darkblue}{rgb}{0, 0, 0.5}
\hypersetup{colorlinks=true, citecolor=darkblue, linkcolor=darkblue, urlcolor=darkblue}

\usepackage[none]{hyphenat}
\usepackage{breakcites}
\usepackage[utf8]{inputenc} %
\usepackage[T1]{fontenc}    %
\usepackage{amsfonts}       %
\usepackage{nicefrac}       %
\usepackage{caption}
\usepackage[font={footnotesize}]{subcaption}
\usepackage{float}
\usepackage{amsmath,amsthm}
\usepackage{amssymb}
\usepackage{mathtools}
\usepackage{soul}
\usepackage{bm}
\usepackage{enumitem}
\setlist{leftmargin=*}
\usepackage{multirow}
\usepackage{wrapfig}
\usepackage{scalerel}
\allowdisplaybreaks
\usepackage{dsfont}
\usepackage[noend]{algpseudocode}
\usepackage{thmtools,thm-restate}
\usepackage{cleveref}  
\usepackage{crossreftools} %

\usepackage{autonum}
\usepackage[ruled,algo2e,linesnumbered,noend]{algorithm2e}%
\makeatletter
\patchcmd{\@algocf@start}%
  {-1.5em}%
  {0pt}%
  {}{}%
\makeatother
\SetAlCapFnt{\footnotesize} %
\SetAlCapNameFnt{\small} %

\usepackage{todonotes}      %

\newtheorem{remark}{Remark}

\newtheorem{definition}{Definition}

\newtheorem{assumption}{Assumption}
\newtheorem{theorem}{Theorem}
\newtheorem{corollary}{Corollary}
\newtheorem{lemma}{Lemma}
\newtheorem{proposition}{Proposition}

\DeclareMathAlphabet{\mathbsf}{OT1}{cmss}{bx}{n}%
\DeclareMathAlphabet{\mathssf}{OT1}{cmss}{m}{sl}%
\newcommand{\cut}[1]{}

\newcommand{\defeq}{\triangleeq}

\newcommand{\para}[1]{\noindent\tbf{#1}\ \ }

\newcommand{\indicator}{\mathbb{I}}

\newcommand{\eventnotag}{\mc{E}}
\newcommand{\event}[1][]{\eventnotag_{#1}}

\newcommand{\x}{\mbi{x}}

\newcommand{\z}{\mbi{z}}

\newcommand{\attention}{\textsc{Attention}}

\newcommand{\kernel}{\kappa}

\newcommand{\rkhs}[1][\kernel]{\mc{H}_{#1}}

\newcommand{\knorm}[1]{\Vert{#1}\Vert_{\kernel}}
\newcommand{\kin}[1][n]{\Vert{\kernel}\Vert_{#1}} %
\newcommand{\kattin}[1][n]{\Vert{\katt}\Vert_{#1}} %

\newcommand{\coreset}[1][j]{\mathcal{S}^{(#1)}}

\SetKwProg{Method}{def}{}{}
\newcommand{\compresstwo}{\hyperref[algo:compress2]{\color{black}{\textsc{Compress2}}}\xspace}

\newcommand{\concat}{\texttt{\textup{concatenate}}}

\newcommand{\brackets}[1]{\left[ #1 \right]}

\newcommand{\parenth}[1]{\left( #1 \right)}

\newcommand{\braces}[1]{\left\{ #1 \right \}}

\newcommand{\abss}[1]{\left| #1 \right |}

\newcommand{\ceils}[1]{\lceil #1  \rceil}
\newcommand{\ceil}[1]{\left\lceil #1 \right \rceil}

\newcommand{\floor}[1]{\left\lfloor #1 \right \rfloor}

\newcommand{\sig}{\sigma}

\def\balign#1\ealign{\begin{align}#1\end{align}}
\def\baligns#1\ealigns{\begin{align*}#1\end{align*}}
\def\balignat#1\ealign{\begin{alignat}#1\end{alignat}}
\def\balignats#1\ealigns{\begin{alignat*}#1\end{alignat*}}
\def\bitemize#1\eitemize{\begin{itemize}#1\end{itemize}}
\def\benumerate#1\eenumerate{\begin{enumerate}#1\end{enumerate}}

\newenvironment{talign*}
 {\csname align*\endcsname}
 {\endalign}
\newenvironment{talign}
 {\csname align\endcsname}
 {\endalign}

\def\balignst#1\ealignst{\begin{talign*}#1\end{talign*}}
\def\balignt#1\ealignt{\begin{talign}#1\end{talign}}
\newcommand{\qtext}[1]{\quad\text{#1}\quad} 
 
\newcommand{\sstext}[1]{\text{\ \ #1\ \ }} 

\let\originalleft\left
\let\originalright\right
\renewcommand{\left}{\mathopen{}\mathclose\bgroup\originalleft}
\renewcommand{\right}{\aftergroup\egroup\originalright}

\def\Holder{H\"older\xspace}

\def\tinycitep*#1{{\tiny\citep*{#1}}}
\def\tinycitealt*#1{{\tiny\citealt*{#1}}}
\def\tinycite*#1{{\tiny\cite*{#1}}}
\def\smallcitep*#1{{\scriptsize\citep*{#1}}}
\def\smallcitealt*#1{{\scriptsize\citealt*{#1}}}
\def\smallcite*#1{{\scriptsize\cite*{#1}}}

\def\mbi#1{\boldsymbol{#1}} %
\def\mbf#1{\mathbf{#1}}
\def\mbb#1{\mathbb{#1}}
\def\mc#1{\mathcal{#1}}

\def\tbf#1{\textbf{#1}}

\def\wtil#1{\widetilde{#1}}
\def\mfk#1{\mathfrak{#1}}

\def\reals{\mathbb{R}} %

\def\naturals{\mathbb{N}} %

\def\<{\left\langle} %
\def\>{\right\rangle}

\def\defeq{\triangleq} %
\def\half{\frac{1}{2}}

\newcommand{\psdge}{\succcurlyeq}

\def\norm#1{\left\|{#1}\right\|} %
\newcommand{\twonorm}[1]{\staticnorm{#1}_2} %
\newcommand{\infnorm}[1]{\staticnorm{#1}_{\infty}} %
\newcommand{\maxnorm}[1]{\staticnorm{#1}_{\max}} %
\newcommand{\opnorm}[1]{\staticnorm{#1}_{\operatorname{op}}} %
\newcommand{\fronorm}[1]{\norm{#1}_{F}} %
\newcommand{\specnorm}[1]{\opnorm{#1}} %

\def\staticnorm#1{\|{#1}\|} %
\newcommand{\inner}[2]{\langle{#1},{#2}\rangle} %
\def\indic#1{\indicator\left[{#1}\right]} %
\def\E{\mbb{E}} %

\def\Esub#1{\E_{#1}}

\newcommand{\Unif}{\textnormal{Unif}}

\ifdefined\nonewproofenvironments\else
\ifdefined\ispres\else
\newenvironment{proof-sketch}{\noindent\textbf{Proof Sketch}
  \hspace*{1em}}{\qed\bigskip\\}
\newenvironment{proof-idea}{\noindent\textbf{Proof Idea}
  \hspace*{1em}}{\qed\bigskip\\}
\newenvironment{proof-of-lemma}[1][{}]{\noindent\textbf{Proof of Lemma {#1}}
  \hspace*{1em}}{\qed\\}
\newenvironment{proof-of-theorem}[1][{}]{\noindent\textbf{Proof of Theorem {#1}}
  \hspace*{1em}}{\qed\\}
\newenvironment{proof-attempt}{\noindent\textbf{Proof Attempt}
  \hspace*{1em}}{\qed\bigskip\\}

\newcommand{\cset}{\mc{S}}
\newcommand{\alg}{\textsc{Alg}\xspace}

\newcommand{\inputcoreset}{\cset_{\textup{in}}}

\newcommand{\halve}{\textsc{Halve}\xspace}

\newcommand{\errsymb}{\ensuremath{\mfk{h}}\xspace}
\newcommand{\compress}{\textsc{Compress}\xspace}

\newcommand{\nin}{n_{\textup{in}}}
\newcommand{\nout}{n_{\textup{out}}}

\newcommand{\halvetag}{\textsc{H}}
\newcommand{\compresstag}{\textsc{C}}

\newcommand{\subsampletag}{{\textsc{S}}}
\newcommand{\expresstag}{{\textsc{E}}}
\newcommand{\wattntag}{{\textsc{W}}}

\newcommand{\ncref}[1]{\cref{#1}: \nameref*{#1}} %
\newcommand{\pcref}[1]{%
  Proof of 
  \texorpdfstring{%
    \ncref{#1}%
  }{%
    \crtcref{#1}: \protect\nameref*{#1}%
  }%
} %

\newcommand{\subg}{\nu}

\newcommand{\V}{\mbf V}

\newcommand{\attout}{\mbi{o}}
\newcommand{\attouthat}{\hat{\mbi{o}}}

\newcommand{\ohyp}[1][n]{\attouthat_{\textrm{hyp},#1}}
\newcommand{\oexp}[1][n]{\attouthat_{\textrm{exp},#1}}
\newcommand{\obal}[1][n]{\attouthat_{\textrm{bal},#1}}

\newcommand{\Xset}{\mathcal{X}}

\renewcommand{\hat}[1]{\widehat{#1}}

\newcommand{\katt}[1][]{\hyperref[eq:katt]{\color{black}{\kernel^{\mathrm{att}}_{#1}}}\xspace}

\newcommand{\key}{\mbi{k}}

\newcommand{\val}{\mbi{v}}

\newcommand{\query}{\mbi{q}}

\newcommand{\subsample}{\hyperref[algo:subsample]{\color{black}{\textsc{Subsample}}}\xspace}

\newcommand{\kh}{\hyperref[algo:khd]{\color{black}{\textsc{KH}}}\xspace}
\newcommand{\khd}{\text{$\kh(\delta)$}\xspace}

\newcommand{\xset}{\mc X}
\newcommand{\xin}{\xset_{\textup{in}}}

\newcommand{\xout}{\xset_{\textup{out}}}

\newcommand{\mbar}{\overline{m}}

\newcommand{\thresh}{\mathfrak{a}}
\newcommand{\multiplier}{\mathfrak{b}}

\newcommand{\mwedge}{m\wedge\mbar}

\newcommand{\express}{\hyperref[algo:express]{\color{black}{\textsc{Express}}}\xspace}
\newcommand{\khexpress}{\hyperref[khexpress-guarantees]{\color{black}{\textsc{KH-Express}}}\xspace}
\newcommand{\ours}{\hyperref[algo:thinformer]{\color{black}{Thinformer Express}}\xspace}
\newcommand{\khexpressd}{\text{$\khexpress(\delta)$}\xspace}
\newcommand{\wattn}{\hyperref[algo:wattn]{\color{black}{\textsc{WtdAttention}}}\xspace}

\newcommand{\oexpress}{\hyperref[algo:express-recursive]{\color{black}{\textup{Offline-\textsc{Express}}}}\xspace}

\newcommand{\ocompresstwo}{\hyperref[algo:compress2-recursive]{\color{black}{\textup{Offline-\textsc{Compress2}}}}\xspace}

\newcommand{\osubsample}{\hyperref[algo:subsample-recursive]{\color{black}{\textup{Offline-\textsc{Subsample}}}}\xspace}

\makeatletter
\newcommand\footnoteref[1]{\protected@xdef\@thefnmark{\ref{#1}}\@footnotemark}
\makeatother

\newcommand{\dstar}[1][j]{\delta_{#1}^\star}

\newcommand{\even}{\textup{even}\xspace}

\newcommand{\subgh}[1][j]{\subg_{\halvetag_{#1}}}
\newcommand{\subge}{\subg_{\expresstag}}

\newcommand{\subgw}{\subg_{\wattntag}}

\newcommand{\psie}{\psi_{\expresstag}}

\newcommand{\sige}{\sig_{\expresstag}}

\newcommand{\stratum}{\Xset_{m,k}}
\newtheorem{myproposition}{Proposition}

\newtheorem{mycorollary}{Corollary}

\newtheorem{mylemma}{Lemma}

\newtheorem{mydefinition}{Definition}

\newtheorem{myremark}{Remark}

\newtheorem{myassumption}{Assumption}

 \crefname{appendix}{App.}{App.}
\crefname{equation}{}{}
\crefname{lemma}{Lem.}{Lems.}
\crefname{theorem}{Thm.}{Thms.}
\crefname{corollary}{Cor.}{Cors.}
\crefname{example}{Ex.}{Exs.}
\crefname{algorithm}{Alg.}{Algs.}
\crefname{algocf}{Alg.}{Algs.} %

\crefname{section}{Sec.}{Secs.}
\crefname{table}{Tab.}{Tabs.}
\crefname{remark}{Rem.}{Rems.}
\crefname{definition}{Def.}{Defs.}
\crefname{Proposition}{Prop.}{Props.}
\crefname{myremark}{Rem.}{Rems.}
\crefname{mylemma}{Lem.}{Lems.}
\crefname{mydefinition}{Def.}{Defs.}
\crefname{myproposition}{Prop.}{Props.}
\crefname{mycorollary}{Cor.}{Cors.}
\crefname{myassumption}{Assum.}{Assums.}
\crefname{figure}{Fig.}{Figs.}
\crefname{enumi}{}{}
\crefname{name}{}{} %

\title{Express Language Modeling}
\author{%
  Albert Gong \\
  Cornell Tech\\
  \texttt{agong@cs.cornell.edu} \\
  \And
  Annabelle Michael Carrell \\
  University of Cambridge \\
  \texttt{ac2411@cam.ac.uk} \\
  \AND
  Raaz Dwivedi \\
  Cornell Tech \\
  \texttt{dwivedi@cornell.edu} \\
  \And
  Lester Mackey \\
  Microsoft Research \\
  \texttt{lmackey@microsoft.com} %
}

\begin{document}

\maketitle

\begin{abstract}
We introduce a new tool, Express, for converting a non-causal attention approximation into a causal approximation with matching approximation guarantees. When combined with the state-of-the-art Thinformer approximation, Express improves upon the best known causal attention guarantees, delivering $\log^{3/2}(n)/s$ approximation error with only $O(s)$ memory and $O(s^2 \log^2(n))$ compression overhead for a sequence of length $n$. We pair these developments with an efficient I/O-aware Triton implementation, demonstrate substantial speedups over FlashAttention 2, and use Express to overcome four resource bottlenecks in the language modeling pipeline: long-context prefill, KV cache compression, long-form memory-constrained decoding, and long-form compute-constrained decoding. 
\end{abstract}

\section{Introduction}
\label{sec:intro}

Attention~\citep{vaswani2017attention} scales quadratically with the sequence length $n$, a cost that is especially acute for language models processing large contexts, generating large token counts, or operating under resource constraints~\citep{bai2024longbench,guo2025deepseek,dao2022flashattention}. 
A variety of practical sub-quadratic attention approximations are now available with strong guarantees for reconstructing the attention output~\citep{zandieh2023kdeformer,han2024hyperattention,carrell2025low,han2025streaming,schroder2026wildcat}, but the majority, including the state-of-the-art {thinning} approximation~\citep{carrell2025low}, only tackle the unmasked attention problem. The harder case, central to language modeling, is \emph{attention with causal masking} where each query attends only over an evolving context of prior key-value pairs. 
This setting governs every stage of language model inference: prefilling long contexts, sequentially decoding new tokens, and managing the key-value cache that grows with each generated token. 

\para{Our contributions} To bridge this gap, we introduce a new meta-procedure, \express, that efficiently turns a high-quality unmasked attention approximation into a high-quality masked approximation. More generally, \express transforms any offline approximation based on thinning  (i.e., subselecting) input points into an efficiently updatable weighted cache with streaming approximation guarantees. 
We pair \express with the state-of-the-art non-causal thinning approximation, Thinformer~\citep{carrell2025low}, and derive causal attention guarantees that improve upon the best in the literature~\citep{han2024hyperattention,han2025streaming}: for bounded inputs, \emph{Thinformer Express} yields $\log^{3/2}(n)/s$ approximation error with only $O(s)$ memory and $O(s^2 \log(s)\log(n/s))$ compression overhead. 
Inspired by FlashAttention 2~\citep{dao2022flashattention,dao2024flashattention}, we then develop an efficient I/O-aware Triton~\citep{tillet2019triton} implementation and 
benchmark \ours in four resource-constrained language modeling settings: 
\begin{enumerate}\itemsep0em
    \item \emph{Long-context prefill}: %
    \ours achieves a $61\times$ speedup over FlashAttention~2~\citep{dao2024flashattention} at 512k tokens and dominates the quality-vs.-runtime tradeoff of the state-of-the-art HyperAttention approximation~\citep{han2024hyperattention}.
    \item \emph{Accelerated KV cache compression}:
    Pairing \ours with each of three leading KV cache compression methods---SnapKV~\citep{li2024snapkv}, StreamingLLM~\citep{xiao2024streamingllm}, and PyramidKV~\citep{cai2025pyramidkv}---reduces total attention runtime on LongBench-E long-context language understanding tasks~\citep{bai2024longbench} without sacrificing accuracy.
    \item \emph{Memory-efficient long-form decoding}:
    On competition-level MATH-500 problems that benefit from step-by-step reasoning~\citep{lightman2023let}, \ours dominates the accuracy-vs.-cache-size tradeoff curves of five leading alternatives, retaining %
    99\% 
    of the accuracy of exact attention with only %
    61\% 
    of the cache size.
    \item \emph{Accelerated long-form decoding}:
    On MATH-500, \ours also dominates the accuracy-vs.-runtime tradeoff curves of five leading alternatives, retaining %
    99\% 
    of the accuracy of exact attention with only %
    55\% 
    of the computation time.
\end{enumerate}

\para{Paper organization}
\Cref{sec:background} provides a formal introduction to thinning algorithms and the causal attention problem.
\Cref{sec:express} presents the \express meta-procedure and \ours, along with their quality, space, and runtime guarantees.
\Cref{sec:experiments} presents our Triton implementation and empirical evaluation, and
\Cref{sec:conclusion} closes with reflections and discussion.

\para{Notation} We define $[n]\defeq\{1,\dots,n\}$,
$\E_{\event}[X] \defeq \E[X\indic{\event}]$ for any probability event $\event$ and integrable random variable $X$, and $(\opnorm{\cdot},\fronorm{\cdot},\maxnorm{\cdot})$ %
as the spectral, Frobenius, and maximum absolute entry  %
of a matrix respectively. 
For kernels $(\kernel,\kernel')$ we write $\kernel\psdge\kernel'$ when $\kernel-\kernel'$ is also a kernel.

\section{Background on Attention and Thinning}\label{sec:background}
\subsection{Unmasked attention}
Unmasked scaled dot-product attention takes as input a sequence of query, key, and value vectors $(\query_i,\key_i,\val_i)_{i=1}^n$ in $\reals^d$ and computes, for each $j\in[n]$, the unmasked attention output \citep{vaswani2017attention} 
\begin{talign}
\label{unmaskedatt}
\attention(\query_j,
           (\key_i, \val_i)_{i=1}^n)
    \defeq
\frac{\frac{1}{n}\sum_{i=1}^{n} \exp(\inner{\query_j}{\key_i}/\sqrt{d})\, \val_i}{\frac{1}{n}\sum_{i=1}^{n} \exp(\inner{\query_j}{\key_i}/\sqrt{d})}, 
\end{talign}
a ratio of two query-specific averages taken over all key-value pairs. 
To avoid the $\Theta(n^2d)$ cost of 
computing each output exactly, 
prior work 
attended only over a relatively small coreset of key-value pairs designed to accurately approximate the  attention output \cref{unmaskedatt} for all queries \citep{carrell2025low}. 
To achieve this, \citet{carrell2025low} employed \emph{sub-Gaussian thinning algorithms} which generate coresets guaranteed to accurately preserve the sorts of averages that arise in the attention computation.
\begin{definition}[\tbf{Sub-Gaussian thinning and halving}]\label{def:alg-subg}
A \emph{thinning algorithm} $\alg$ takes a dataset $\xin = (\x_i)_{i=1}^{\nin}$ as input and outputs a weighted coreset $\xout$ of $\nout$ points $\x\in\xin$ paired with positive weights $w$ that sum to $1$. When $\nout=\floor{n/2}$ and all weights are equal, we call \alg a \emph{halving algorithm}. 
We say $\alg$ is \emph{$(\kernel,\subg,\delta)$-sub-Gaussian} given $\xin$ %
if $\subg \ge 0$, $\delta \in [0,1)$, 
$\kernel$ is a kernel with reproducing kernel Hilbert space (RKHS) $\rkhs$ and norm $\knorm{\cdot}$ \citep[Def.~4.18]{Steinwart2008SupportVM}, 
and there exists an event $\mc E$ with probability at least $1 - \delta/2$ such that
\begin{talign}\label{eq:subg_def}
\Esub{\event}\brackets{\exp\parenth{ \frac{1}{\nin}\sum_{i=1}^{\nin} f(\x_i) - \sum_{(\x,w)\in\xout} w\,f(\x)}} \leq \exp\parenth{\frac{\subg^2}{2} \knorm{f}^2}
    \qtext{for all} 
f \in \rkhs.
\end{talign}
For any event $\event$, we say \alg is \emph{$(\kernel,\subg)$-sub-Gaussian on $\event$} given $\xin$ if \cref{eq:subg_def} holds.
\end{definition}
The sub-Gaussian constant $\subg$ controls the error of a thinning algorithm, ensuring that, with high probability,  expensive averages over the input are $O(\subg)$-close to cheap averages over the coreset. 

Perhaps the most familiar sub-Gaussian thinning algorithm is uniform subsampling: sampling $\nout$ equal-weighted points without replacement is $(\kernel,\subg,0)$-sub-Gaussian for any kernel $\kernel$ with $\subg=\sqrt{{\kin}/{\nout}}$ and $\kin\defeq \max_{i\in[n]}\kappa(\x_i,\x_i)$ \citep[Prop.~B.1]{carrell2025low}. 
However, uniform summaries are not especially concise as they require $\nout=10000$ points for $1\%$ relative error \citep[Prop.~1]{carrell2025low}.
Fortunately, substantially more accurate thinning algorithms are also available. For example, kernel halving~\citep[\khd,][]{dwivedi2024kernel,dwivedi2021generalized,carrell2025low}  achieves $(\kernel,\subg,\delta)$-sub-Gaussianiaty for $\subg = \sqrt{\kin\log(2\nin/\delta)}/\nout$ by balancing kernel averages using non-uniform randomness.

The kernel $\kernel$ in \cref{def:alg-subg} dictates which types of averages are approximated well. 
\citet[Thm.~2]{carrell2025low} introduced a key-value kernel that mimics the summands in the attention mechanism \cref{unmaskedatt}, 
\begin{talign}\label{eq:katt}
\katt[\xset]((\key,\val),(\key',\val')) \defeq \exp(\inner{\key}{\key'}/\sqrt{d})(1+{\inner{\val}{\val'}/}{\max_{(\key,\val)\in\xset}\infnorm{\val}^2}), 
\end{talign}
and showed that thinning the set of key-value pairs $\xset_n = (\key_i,\val_i)_{i=1}^n$ with $\katt[\xset_n]$ 
provides an $O(\subg)$-accurate approximation to the unmasked attention output \cref{unmaskedatt} for all queries simultaneously. Our aim is to extend this strong guarantee to the more challenging causal masking setting that arises in language modeling.

\subsection{Attention with causal masking}
Scaled dot-product attention with causal masking computes, for each $j\in[n]$, the masked  output~\citep{vaswani2017attention}
\begin{talign}
\label{def:att}
\attout_j
    &\defeq
\attention(\query_j,
           (\key_i, \val_i)_{i=1}^j)
    \defeq
\frac{\frac{1}{j}\sum_{i=1}^{j} \exp(\inner{\query_j}{\key_i}/\sqrt{d})\, \val_i}{\frac{1}{j}\sum_{i=1}^{j} \exp(\inner{\query_j}{\key_i}/\sqrt{d})},
\end{talign}
where each query $\query_j$ now attends only over key-value pairs with indices $i\leq j$. 
Two variants of this problem commonly arise in language modeling. In the \emph{prefill} or context-processing phase, we observe and hence can process all $n$  query-key-value tuples at once, as each corresponds to a token in a fully observed context sequence. In the even more challenging \emph{decoding} or generation phase, we must generate outputs incrementally as each tuple $(\query_{j},\key_{j},\val_{j})$ is generated on-the-fly based on the prior output scores $(\attout_i)_{i=1}^{j-1}$. In this setting, even the sequence length $n$ may not be known in advance. 
In either case, computing all $n$ outputs $(\attout_i)_{i=1}^n$ using standard matrix multiplication requires $\Theta(d\,n^2)$ computation, a cost that quickly becomes prohibitive as $n$ grows. Moreover, by default, in the decoding phase, all prior keys and values are stored in a KV cache to allow future tokens to attend over them. This $\Theta(d\,n)$ memory cost also quickly becomes prohibitive for models deployed on resource-constrained devices. 

To overcome both resource bottlenecks, we would ideally exploit sub-Gaussian thinning once again to generate a small key-value coreset suitable for approximating each attention output \cref{def:att}. 
However, the causal mask and the incremental constraint of decoding imply that a single coreset will no longer suffice. Instead we will generate a sequence of coresets $(\xout(j))_{j=1}^n$ and use $\xout(j)$ to form a weighted attention approximation of the $j$-th attention output:
\begin{talign}
\label{def:att_approx}
\attouthat_j
    \defeq
\frac{\sum_{((\key,\val),w) \in \xout(j)} w\, \exp(\inner{\query_j}{\key}/\sqrt{d})\, \val}{\sum_{((\key,\val),w) \in \xout(j)} w\, \exp(\inner{\query_j}{\key}/\sqrt{d})}.
\end{talign}
The key challenge lies in efficiently updating each coreset in a manner that controls coreset size, memory usage, runtime, and output approximation error across the entire sequence of inputs. 
To achieve this, we introduce a new meta-procedure, \express, that efficiently converts any sub-Gaussian halving algorithm into a streaming sub-Gaussian thinning algorithm suitable for causal attention approximation.
\section{Express Language Modeling}\label{sec:express}
\subsection{Express }
\begin{wrapfigure}{r}{0.4\textwidth}
    \vspace{-4\baselineskip}
    \includegraphics[trim={.5cm 1.3cm .5cm 1.1cm}, clip=true, width=\linewidth]{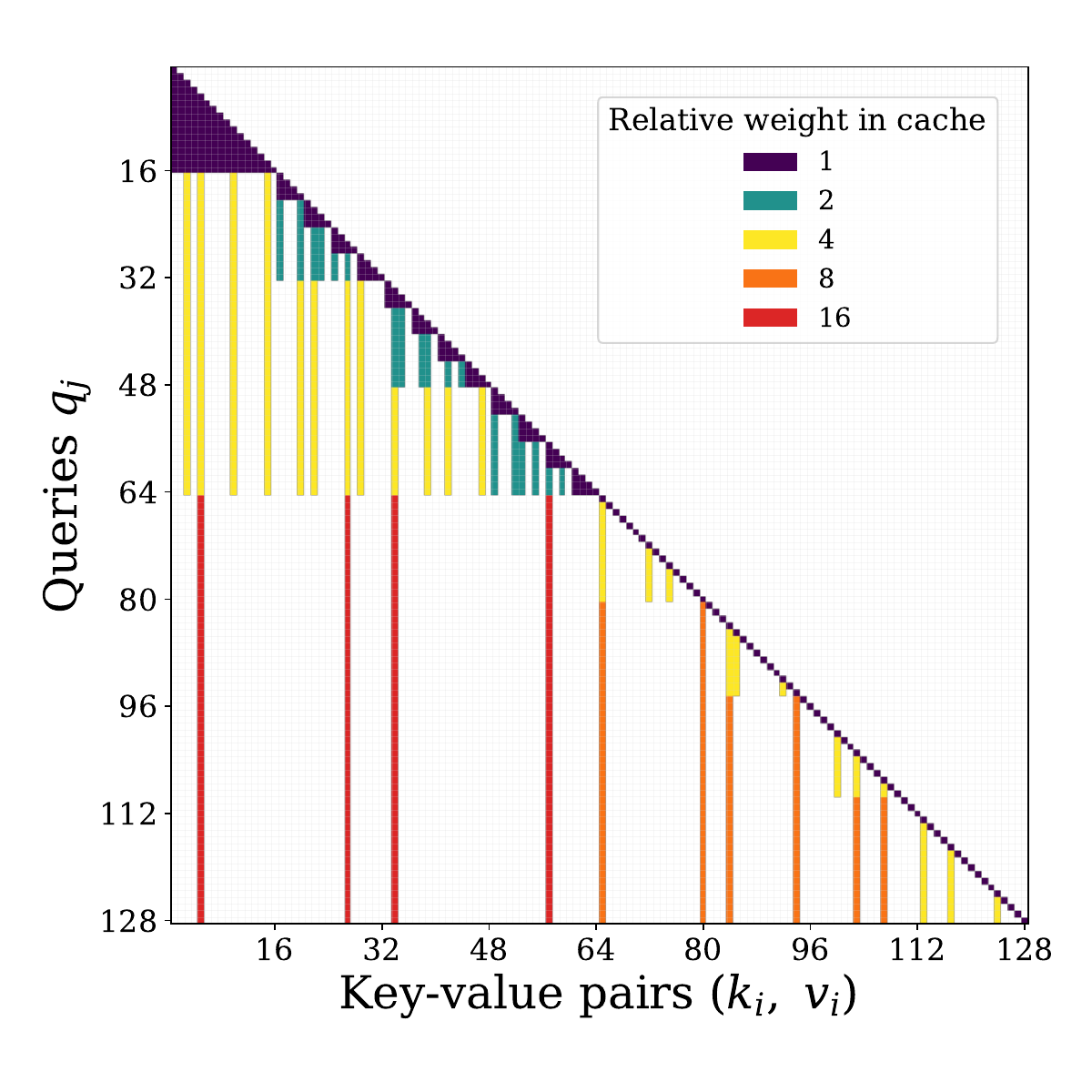}
    \vspace{-2\baselineskip}
\end{wrapfigure}

\express (\cref{algo:express}) can be viewed as an updatable weighted cache object that is initialized with a target cache size $\nout$, an inflation factor $\mbar$, and a halving algorithm $\halve$. 
Initially, the cache is empty, and updates proceed in three phases. 
In the exact phase, the first $\nout$ points processed are stored directly in the \express coreset $E$. 
In the thin phase, points are batched into groups of size $2^m\nout$ and incrementally thinned down to size $\nout$ before entering $E$. 
Whenever the size of $E$ reaches $4\nout$, we thin its size back down to $\nout$ by applying \halve twice (the \halve phase) and increase the thinning factor $m$.

In the thin phase, one could imagine directly applying \halve to carry out the thinning, but high-quality halving algorithms like \khd have a quadratic complexity that would render the update too slow. Instead we develop a more efficient strategy: we first uniformly \subsample from size $2^m\nout$ to $2^{\mbar}\nout$ and then incrementally thin from $2^{\mbar}\nout$ to $\nout$ using \compresstwo, a generalization of the \compress algorithm of \citet{shetty2022distribution} that reduces thinning time while preserving quality by only applying \halve to small groups of points with geometrically increasing size. 
Throughout, we choose the inflation factor $\mbar \in \naturals$ so that $2^{\mbar}=\Theta(\nout)$ and $2^{\mbar-1}$ divides $\nout$ evenly.  
Our first theorem, proved in \cref{proof:express-guarantees}, bounds the maximum cache size, space complexity, and runtime of \express in terms of its halving algorithm.\footnote{We provide a complementary memory bound for a recursive offline version of \express in \cref{app:express_recursive_storage_cost}.}

\newcommand{\keep}{\textup{keep}}
\newcommand{\wtdcache}{\textup{wtd\_cache}}
\newcommand{\update}{\textup{update}}
\begin{algorithm2e}[t!]
\DontPrintSemicolon
\caption{\small $\express(\nout, \mbar, \halve)$: Streaming compression w/ target cache size $\nout$, inflation factor $\mbar$} %
\label{algo:express}
\small
\BlankLine
// Initialize the thinning factor $m$, datapoint counter $n$,  \compresstwo counter $\ell$, and empty \express coreset \\
$m \gets 0$;\quad $n \gets 0$;\quad $\ell \gets 0$; \quad $E \gets \braces{}$ \\
    // Initialize thinning objects: $S$ thins from size $2^m\nout$ to $2^{\mwedge}\nout$; $C$ thins  from size $2^{\mwedge}\nout$ to $\nout$ \\ %
    $S \gets \subsample(m,\mbar)$;\quad
    $C \gets \compresstwo(\mwedge, 2^{\mwedge}\nout, \halve_m)$\;
\BlankLine
\Method{\update$(\x)$:}{
    {// Exact phase: store initial $\nout$ points in \express coreset} \\
    $n \gets n + 1$;\quad \lIf{$n \leq \nout$}{$E.\textup{append}(\x)$; \textbf{return}}
    \BlankLine
    // Thin phase: thin next $2^m\nout$ points down to $\nout$ with \subsample + \compresstwo \\
    $\ell \gets \ell + 1$;\quad
    \lIf*{$S.\keep(\x)$}{$C.\update(\x)$;}
    \quad\lIf{$\ell == 2^m\nout$}{$E\gets E\cup C.\cset_{m\wedge\mbar}$; $\ell \gets 0$}\label{line:express-store-compress2}
    \BlankLine
    // \halve phase: \halve  \express coreset twice to reach size $\nout$ and update thinning factor $m$\\ 
    \lIf{$n == 4\cdot 2^m \nout$}{
        $E \gets \halve_{m,m}(\halve_{m,m}(E))$;\quad $m\gets m+2$\label{line:express-warm-halve}
        }
    \BlankLine
    // Reset thinning objects \\
    \lIf{$\ell == 0$}{
        $S \gets \subsample(m,\mbar)$;\quad 
        $C \gets \compresstwo(m \wedge \mbar, 2^{m \wedge \mbar}\nout, \halve_m)$}
}
\BlankLine
\Method{\wtdcache$()$:}{
    \KwRet{$C.\wtdcache().\concat((E, 1))$}
    \quad // Cache of coresets with associated weights
}
\end{algorithm2e}
\begin{figure}[t!]
\begin{minipage}[t]{0.48\textwidth}
\begin{algorithm2e}[H]
\DontPrintSemicolon
\caption{\small $\compresstwo(m, \nin, \halve)$}
\label{algo:compress2}
\small
// \halve points in groups of size $\big(\frac{2^i \nin }{4^{m-1}}\big)_{i=0}^m$
\BlankLine
// Initialize counter and coresets for each halving level \\
$n \gets 0$;\quad \lFor{$i = 0, 1, \ldots, m$}{$\cset_i \gets \braces{}$}
\BlankLine
\Method{\update$(\x)$:}{
    $n \gets n + 1$;\quad 
    $\cset_0 \gets \cset_0 \cup \braces{\x}$\;
    \For{$i = 0, \ldots, m-1$}{
        \If{$\abss{\cset_i} == \nin \cdot 2^{i} / 4^{m-1}$}{  \label{line:cset-i}
            $\!\!\cset_{i+1} \!\gets\! \cset_{i+1} \cup \halve_i(\cset_i)$;\label{line:compress2-halve}
            $\cset_i\!\gets\! \braces{}$\;
        }
    }
}
\BlankLine
\Method{\wtdcache$()$:}{    \KwRet{$[(\cset_i,\, 2^{i-m})]_{i=0}^{m}$}\quad // Weighted cache
}
\end{algorithm2e}
\end{minipage}%
\hfill
\begin{minipage}[t]{0.48\textwidth}
\begin{algorithm2e}[H]
\DontPrintSemicolon
\caption{\small \textsc{Subsample}($m, \mbar$)}
\label{algo:subsample}
\small
// Keep one uniform index from each group of $2^{m-\mbar}$\\
$n \gets 0$;\quad $i \gets 0$\;
\BlankLine
\Method{\textup{keep}$(\x)$:}{
    $n \gets n + 1$\;
    \lIf{$m \leq \mbar$}{\textbf{return}~\texttt{True}}
    \If{$(n-1) \bmod 2^{m - \mbar} == 0$}{
        $i\gets \Unif([n, n + 2^{m - \mbar}))$\;
    }
    \KwRet{$n == i$}\;
}
\end{algorithm2e}
\end{minipage}
\end{figure}
\begin{theorem}[\express memory and runtime]
\label{express-guarantees}
Consider running \express (\cref{algo:express}) with  $n$ updates. If \halve with input size $\nin$ has space complexity $s_{\halvetag}(\nin)$ and runtime $r_{\halvetag}(\nin)$, then \express has 
weighted cache size at most $6\nout$, space complexity at most $6\nout d + s_{\halvetag}(4\nout)$, and 
runtime 
\begin{talign}\label{eq:express-runtime}
r_{\expresstag}(n) 
    &\leq 
(
r_{\halvetag}(4 \nout) + r_{\halvetag}(2 \nout)
    +
3r_{\compresstag}\parenth{\mbar}
    +
2^{\mbar}3\nout)\,\ceils{\log_4(n/\nout)}\,\indic{n\geq 4\nout}
\end{talign}
where %
$r_{\compresstag}(m) \defeq \sum_{j=0}^{m-1} 4^{j} r_{\halvetag}\parenth{2^{1-j}\nout }$ is the runtime of $\compresstwo(m,2^m\nout,\halve)$.
\end{theorem}

Notably, the maximum cache size  of \express is $6\nout$, a quantity
\emph{independent} of the sequence length $n$. 
Moreover, since $\mbar\leq \log_2(2\nout)$, the \express runtime \cref{eq:express-runtime} has at most a logarithmic explicit dependence on $n$.
Let us now unpack the implications of the runtime bound in \cref{eq:express-runtime} for a halving algorithm with polynomial runtime.

\begin{corollary}[\express runtime with polynomial-time \halve]
\label{rmk:express-polynomial-runtime}
Under the conditions of \cref{express-guarantees}, 
if $r_{\halvetag}(\nin)\leq\rho\nin^\tau$ for some $\tau\geq 0$ and all $\nin\in[4\nout]$ 
then 
\begin{talign}\label{eq:express-polynomial-runtime}
r_{\expresstag}(n)
    &=
\begin{cases}
O\parenth{\rho\,\nout^2\log(n/\nout)\log(\nout)} 
    & \text{if } \tau=2, \qtext{and}\\
O\parenth{\rho\, 4^\tau\frac{|(2\nout)^\tau-(2\nout)^{2}|}{|2^\tau-2^{2}|}\log(n/\nout)}
    & \text{if } \tau\neq 2.
\end{cases}
\end{talign}
\end{corollary}

\cref{rmk:express-polynomial-runtime}, proved in \cref{proof:express-polynomial-runtime}, also reveals the runtime benefits of \express when $\nout$ grows with $n$. 
For example under heavy square-root compression with $\nout = \sqrt{n}$, \express converts a non-causal quadratic-time halving algorithm into a causal, near-linear $O(n\log^2(n))$ time thinning algorithm. Similarly, \express offers a quadratic speed-up (from $\Theta(n^\tau)$ to $O(n^{\tau/2}\log n)$) for super-quadratic halving algorithms with $\tau > 2$ like Gram-Schmidt Thinning \citep{carrell2025low}. 

Our next theorem, proved in \cref{proof:express-quality}, bounds the sub-Gaussian constants of \express in terms of its halving algorithm.
\begin{theorem}[Quality of \express]\label{express-quality}
For each $n\in\naturals$, let $\xout(n)$ denote the weighted coreset obtained by updating \express (\cref{algo:express}) with $\xset_n\defeq(\x_i)_{i=1}^n$, 
and suppose each $\halve_{m,i}$ call invoked within the first $n$ updates is $(\kernel,\subgh[n]\!(\nin),\delta_{m,i})$-sub-Gaussian with $\nin\subgh[n]\!(\nin)$ non-decreasing in $\nin$ and $n$.
Then each $\xout(n)$ is $(\kernel,\subg_{\expresstag}(n))$-sub-Gaussian given $\xset_n$ with 
$m_n \defeq 2\ceils{\log_4(n/(4\nout))}$ and
\begin{talign}
\label{eq:express-subg}
\subg_{\expresstag}^2(n)
    &\leq
\frac{16\parenth{16+3(m_n\wedge\mbar)}}{15}
\,(\subgh[n]^2\!(4\nout) + \subgh[n]^2\!(2\nout))\,\indic{n\geq 4\nout}
    +
\frac{16\kin[n]\indicator\brackets{m_n > \mbar}}{5\nout2^{\mbar}} 
\end{talign}
on an event $\event[n]\supseteq\dots\supseteq\event[1]$ of probability 
$1-\sum_{\textup{even } m=0}^{m_n}({\delta_{m,m}+\frac{3}{8}\sum_{i=0}^{\parenth{m\wedge\mbar}-1}4^{\parenth{m\wedge\mbar}-i}\delta_{m,i}})$. 
\end{theorem}
\newcommand{\expressinflation}{2\sqrt{\log_2(\nout)+6}}
An important implication is that non-causal to causal conversion with \express inflates sub-Gaussian error by at most a factor of $\sqrt{\frac{16}{15}\parenth{16+3\mbar}}\leq \expressinflation$ relative to a single \halve phase.

In our experiments, we instantiate \express with the kernel halving algorithm \citep[\khd,][]{dwivedi2024kernel,dwivedi2021generalized,carrell2025low}, which has quadratic runtime and $s_{\halvetag}(\nin) = \nin d$ space complexity.
The following corollary, proved in \cref{proof:khexpress-guarantees}, bounds the runtime, space, and quality of the resulting algorithm.

\begin{corollary}[\khexpressd guarantees]\label{khexpress-guarantees}
For $\delta\in(0,1)$ and a kernel $\kernel_{\xin}$ that can vary with the choice of input data $\xin$, 
define $\khexpress(\delta)$ as \express with  
$\halve_{m,i}(\xin)\defeq\kh(\delta_{m,i},\kernel_{\xin})$, 
$\delta_m\defeq \frac{\delta}{2}(\frac{1}{\log_2(m/2+2)}-\frac{1}{\log_2(m/2+3)})$, 
$\delta_{m,m} \defeq \frac{\delta_m}{2}$, 
and $\delta_{m,i}\defeq \frac{4^{i+1-(m\wedge\mbar)}\delta_m}{3(m\wedge\mbar)}$ otherwise.
Updating $\khexpress(\delta)$ $n$ times requires at most $10\nout d$ space and
\begin{talign}\label{eq:khexpress-runtime}
    r_{\expresstag}(n)
    =
O(d\nout^2\log({n/}{\nout})\log(\nout))
\end{talign}
runtime.
Moreover, if  $\kernel_{\xset}\psdge\kernel_{\xset'}$ when $\xset\subseteq\xset'$, then, for all $n\in\naturals$, the weighted coreset $\xout(n)$ obtained by updating $\khexpress(\delta)$ with $\xset_n\defeq(\x_i)_{i=1}^n$ is $(\kernel_{\xset_n},\subg_{\expresstag}(n))$-sub-Gaussian with 
\begin{talign}\label{eq:khexpress-error}
\subg_{\expresstag}^2(n)
    &\leq
\frac{4\parenth{16+3(m_n\wedge\mbar)}\indic{n\geq 4\nout}}{3}
\frac{a_n\log(5\nout\mbar(m_n+5)\log_2^2(m_n/2+3)4^{\mbar}/\delta)}{\nout^2} 
    +
\frac{16a_n\indicator\brackets{m_n > \mbar}}{5\nout2^{\mbar}} \\
    &=
O\parenth{\frac{a_n\log(\nout)\log(\nout\log(\frac{n}{\nout})/\delta)}{\nout^2}
    +
\frac{a_n}{2^{\mbar}\nout}}
    \qtext{for}
a_n 
    \defeq 
\max_{\xin\subseteq\xset_n, \x\in\xin}  \kernel_{\xin}(\x,\x)
\end{talign}
and $m_n \defeq 2\lceil{\log_4(\frac{n}{4\nout})}\rceil$ 
on a common event $\event$ of probability $1-{\delta/}{2}$.
\end{corollary}

\subsection{Thinformer Express}

\begin{figure}[htb]
\begin{minipage}[t]{0.48\textwidth}
\begin{algorithm2e}[H]
\DontPrintSemicolon
\caption{\small Thinformer Express\,($\errsymb, \mbar, \delta$)}
\label{algo:thinformer}
\small
\BlankLine
// Initialize \khexpressd with attention kernel \cref{eq:katt}\!\!\!\\
$\halve_{m,i}(\xin) \defeq \kh(\delta_{m,i},  \katt[\xin])$ \sstext{for} $m,i\geq 0$\\
$C \gets \express(2^{\errsymb}, \mbar,\halve)$\;
\BlankLine
\Method{\textup{attend}$(\query_n,\key_n,\val_n)$:}{
    $\attouthat_n \gets \wattn((\query_n,\key_n,\val_n), C)$\;
    $C.\update((\key_n,\val_n))$ \\
    \KwRet{$\attouthat_n$}\;
}
\end{algorithm2e}
\end{minipage}%
\hfill
\begin{minipage}[t]{0.48\textwidth}
\begin{algorithm2e}[H]
\DontPrintSemicolon
\caption{\textsc{WtdAttention}\label{algo:wattn}}
\small
  \KwIn{%
  Token $(\query,{\key},{\val})$, weighted cache object $C$}
  \BlankLine
  $[(\cset_i,w_i)]_{i=0}^L\gets C.\wtdcache()$ \quad// Unpack cache \\   \BlankLine
  // Attend over input token as well \\
  $\cset_{-1} \gets \{({\key},{\val})\}$;\quad $w_{-1}\gets 2^{-C.m}$ \\
  \BlankLine  \KwRet{$\frac{\sum_{i=-1}^{L}\sum_{(\wtil{\key},\wtil{\val}) \in \cset_i} w_i \exp({\inner{\query}{\wtil{\key}}}{/\sqrt{d}})\, \wtil{\val}}{\sum_{i=-1}^{L}\sum_{(\wtil{\key},\wtil{\val}) \in \cset_i} w_i \exp({\inner{\query}{\wtil{\key}}}{/\sqrt{d}})}$}
\end{algorithm2e}
\end{minipage}
\end{figure}
Next, we introduce \ours  (\cref{algo:thinformer}), our solution to the causal attention approximation problem. \ours first initializes a \khexpressd weighted cache object $C$ with the key-value attention kernel $\katt[\xin]$ \cref{eq:katt} and a target cache size $\nout=2^{\errsymb}$.
For each new token $(\query_n,\key_n,\val_n)$, \cref{algo:thinformer} uses \wattn to attend only over $(\key_n,\val_n)$ and the key-value pairs in $C$ and finally updates the weighted cache with the latest key-value pair. 
\cref{khexpress_query_cost} shows that the total \emph{query time} of \cref{algo:thinformer}, i.e., the cost of running \wattn atop the weighted cache $C$, is only $O(\nout nd)$, a significant improvement over $\Theta(n^2d)$-time exact attention whenever the target cache size $\nout \ll n$. 
Remarkably, by \cref{khexpress-guarantees}, the \emph{compression time} of \cref{algo:thinformer}, i.e., the cost of maintaining the weighted cache, is even smaller: $O(d\nout^2\log(\nout)\log(n/\nout))$ with only a logarithmic explicit dependence on the sequence length $n$. 
Our final result, proved in \cref{proof:att-err}, provides a strong attention approximation guarantee for \ours. 
\begin{theorem}[Quality of Thinformer Express]\label{att-err}
With probability at least $1-\delta$, \ours (\cref{algo:thinformer}) with $\nout\defeq 2^{\errsymb}$ and $\V_n \defeq [\val_i]_{i=1}^n$ satisfies $\attouthat_n=\attout_n$ for all $n\leq4\nout$ and, for all $n>4\nout$,
\begin{talign}
\infnorm{\attouthat_n - \attout_n}
    &\leq
2\rho_n
\sqrt{2(d+1)\log({2(d+1)/}{\delta_n})} \\ \label{eq:att-err}
&\quad\times\sqrt{\frac{4\parenth{16+3(m_n\wedge\mbar)}}{3}
\frac{\log(5\nout\mbar(m_n+5)\log_2^2(m_n/2+3)4^{\mbar}/\delta)}{\nout^2}
    +
\frac{16\indicator\brackets{m_n > \mbar}}{5\nout2^{\mbar}}}\\
    &=
O\left(\frac{\rho_n\sqrt{d\log(dn/\delta)\log(\nout)\log(\nout\log(\frac{n}{\nout})/\delta)}}{\sqrt{2^{\mbar}\nout}}\right), \qtext{for} 
\end{talign}
$m_n\defeq 2\ceils{\log_4(\frac{n}{4\nout})}$, 
$\delta_n \defeq \frac{\delta\log_2(1+1/(n+1))}{2\log_2(n+1)\log_2(n+2)}$, and
$\rho_n \defeq \frac{\exp\big((\twonorm{q_n}^2+\max\limits_{i\in[n]}\twonorm{k_i}^2)/(2\sqrt{d})\big)\cdot\maxnorm{V_n}}{\frac{1}{n}\sum_{i\in[n]}\exp(\inner{q_n}{k_i}/\sqrt{d})}.$
\end{theorem}

Recently, \citet[Thm.~2]{liberty2026nearly} showed that any coreset-based approximation must suffer at least $\Omega(\frac{\exp({R_n^2/\sqrt{d}})\cdot\maxnorm{\V_n}}{\nout})$ %
error for some input stream where $ R_n\defeq\max_{i\in[n]}\max(\twonorm{\query_i},\twonorm{\key_i})$. 
Since the \cref{att-err} prefactor $\rho_n$ can be conservatively upper bounded by $e^{2R_n^2/\sqrt{d}}\cdot\maxnorm{\V_n}$ and $R_n$ is uniformly bounded for any transformer-based language model with a finite vocabulary \citep[Thm.~2.2]{velivckovic2024softmax}, \cref{att-err} recovers the minimax-optimal rate of attention approximation, up to log factors. %
In practice, $\rho_n$ can be substantially smaller than its conservative upper bound. For example, on the first sequence of length $n=32768$ from the LongBench-E experiment of \cref{sec:experiments}, the initial attention layer of the Chameleon-7B language model \citep{team2024chameleon} has a median $e^{2R_n^2/\sqrt{d}}\cdot\maxnorm{\V_n}$ value of $1.72 \times 10^{13}$ and a much more reasonable median $\rho_n$ value of $7683$.

It is also instructive to compare \cref{att-err} with the best guarantees for practical causal attention approximation in the literature. %
HyperAttention~\citep[Thm.~1, Lems.~1-2]{han2024hyperattention} uses a combination of independent key-value subsampling, locality sensitive hashing (to identify large output entries), and divide-and-conquer recursion to approximate causal attention in the prefill stage. 
Given $2n$ tokens, HyperAttention with only causal masking %
provides the following guarantee for the $n$-th causal attention output with $O(dn^{1+a})$ runtime and probability at least $\half$:
\begin{talign}\label{hyp-gamma-bounded-guarantee}
\infnorm{\ohyp[n]-\attout_{n}} &= O\Big(\frac{\exp(\frac{17}{3}R_n^2) \,(\log n)^{\frac{1}{6}}}{n^{a/6}}\cdot\specnorm{\V_n}\Big),
\end{talign}
as this estimate coincides with the unmasked approximation analyzed in \citep[Sec.~J.5]{schroder2026wildcat}.
Meanwhile, under the same conditions, \ours with $2^{\mbar} = \Theta(\nout)$ and $\nout=n^a$ guarantees 
\begin{talign}\label{exp-gamma-bounded-guarantee}
\max_{j\in[n]}\infnorm{\oexp[j]-\attout_j} \leq O\big(\frac{\exp(2R_n^2)\log n\sqrt{d\log(dn)}}{n^{a}}\cdot\maxnorm{\V_n}),
\end{talign}
by \cref{att-err,khexpress_query_cost,khexpress-guarantees}. 
This new guarantee provides several improvements over the former: a substantially faster error decay rate ($n^{-a}$ vs.\ $n^{-a/6}$) for a given sub-quadratic runtime $dn^{1+a}$, a milder dependence on the norm parameter $R_n$, and an improved dependence on the value matrix $\V_n$ as $\opnorm{\V_n}$ can be as large as $\sqrt{nd}\,\maxnorm{\V_n}$.

To approximate causal attention in a streaming fashion, BalanceKV~\citep[Thm.~3.1]{han2025streaming} uses the quadratic-time self-balancing walk of \citet{alweiss2020discrepancyminimizationselfbalancingwalk} to select key-value coresets and the classical merge-reduce technique \citep{matousek1995approximations,chazelle1996linear,phillips2008algorithms} to reduce time and memory requirements. 
In the notation of \cref{att-err},
BalanceKV %
with batch size $\nout$ guarantees \citep[Sec.~J.3]{schroder2026wildcat}
\begin{talign}
\max_{j\in[n]}\infnorm{\obal[j]-\attout_j}
\leq O\left(\frac{\exp({2R_n^2}{/\sqrt{d}})\log(n/\nout)\sqrt{d}\log(dn)}{\nout}\cdot\fronorm{\V_n}\right)
\end{talign}
with high probability using $\Theta(d\nout\log(n/\nout))$ space, 
$\Theta(dn\nout\log(n/\nout))$ query time,
and
$\Theta(dn\nout\log(n/\nout))$ compression time. 
\ours with $2^{\mbar}=\Theta(\nout)$ introduces several improvements over this state-of-the-art guarantee. 
First, for the same fixed cache size $\nout$, \ours offers a space complexity independent of $n$ and a query complexity improved by a factor of $\log(n/\nout)$.
Second, our \ours error analysis yields an improved dependence on the value matrix $\V_n$ as $\fronorm{\V_n}$ can be as large as $\sqrt{nd}\,\maxnorm{\V_n}$. 
Finally, the $\Theta(dn\nout\log(n/\nout))$ compression overhead of BalanceKV is substantially larger than the $\Theta(d\nout^2\log(\nout)\log(n/\nout))$ overhead of \ours.
For example, selecting a heavily compressed cache size of $\nout=\sqrt{n}$ introduces only $O(dn\log^2(n))$ compression overhead for \ours but order $dn^{3/2}\log(n)$ overhead for BalanceKV. 

Like prior work, \cref{att-err} provides a guaranteed attention approximation for the input stream observed by exact attention, precisely the setting encountered when processing long contexts with an initial attention layer.  In \cref{sec:experiments}, we will evaluate Thinformer Express both in this setting and in even more challenging settings with off-target inputs due to multilayer approximation and sequential decoding.

\section{Experiments}\label{sec:experiments}

\begin{figure}[t!]
\centering
\begin{subfigure}[b]{0.49\textwidth}
\caption{\textbf{Masked prefill}\label{fig:masked-prefill}}
\centering
\includegraphics[width=\textwidth]{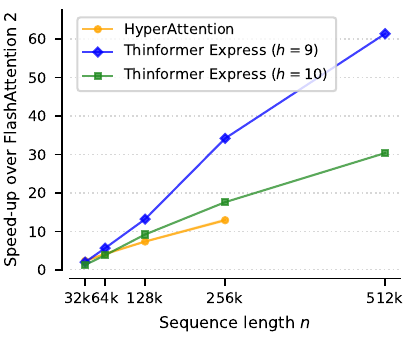}
\end{subfigure}
\hfill
\begin{subfigure}[b]{0.49\textwidth}
\captionsetup{margin={0.5cm, 0.0cm}}
\caption{\textbf{Masked prefill ($n=32768$)}\label{fig:masked-prefill-32k}}
\centering
\includegraphics[width=\textwidth]{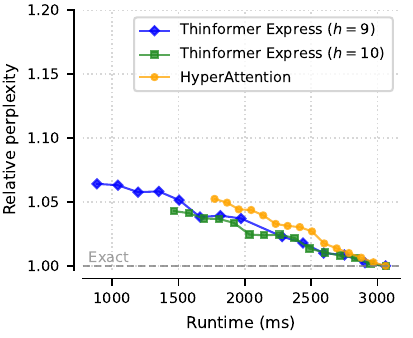}
\end{subfigure}
\begin{subfigure}[b]{0.49\textwidth}
\caption{\textbf{Attention approximation error}\label{fig:maxnorm}}
\centering
\includegraphics[width=\textwidth]{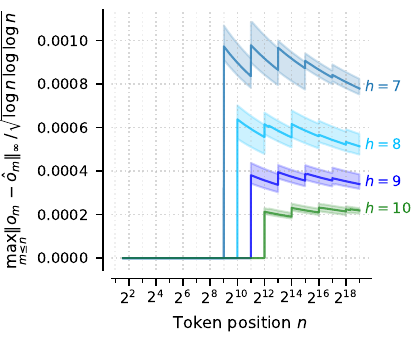}
\end{subfigure}
\hfill
\begin{subfigure}[b]{0.49\textwidth}
\caption{\textbf{Unmasked prefill}}
\label{fig:unmasked-prefill}
\centering
\includegraphics[width=\textwidth]{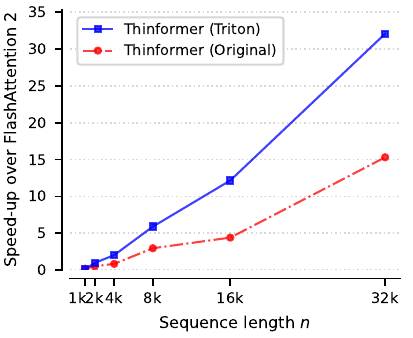}
\end{subfigure}
    \caption{\tbf{Accelerating long-context prefill.} For masked attention, \ours attains larger speed-ups over FlashAttention 2 than the Triton HyperAttention algorithm (a) while improving the runtime-perplexity trade-off on LongBench-E tasks (b). (c) As predicted by \cref{att-err}, the scaled Thinformer Express attention approximation error $\infnorm{\attouthat_n - \attout_n} / \sqrt{\log n \log\log n}$ remains bounded for all $n$ and decays with increasing cache size $\nout=2^\errsymb$. (d) For unmasked attention, our I/O-aware Triton implementation also provides speed-ups over the original Torch-compiled Thinformer.}
    \label{fig:context-runtime}
\end{figure}

We next complement our methodological and theoretical development with an empirical evaluation across a suite of standard approximate attention benchmarks.  We consider four potential language modeling applications of \ours---accelerating long-context prefill, accelerating KV cache compression, memory-constrained long-form decoding, and accelerating long-form decoding---and in each case demonstrate its benefits over existing solutions. We release our open-source code at \url{https://github.com/microsoft/thinformer} and provide supplementary experiment details in \cref{sec:experiment_supplement}. 
Unless otherwise noted, all results were obtained with the default \ours settings $\mbar = 2\floor{\errsymb/2}$ and $\delta=\half$.

\para{An efficient I/O-aware Triton implementation}\label{sec:triton}
To translate the favorable runtime and space complexity of \cref{khexpress-guarantees,khexpress_query_cost} into real-world gains, we implement \ours using Triton \citep{tillet2019triton}.
Specifically, we write the core \kh (\cref{algo:khd}) and \wattn (\cref{algo:wattn}) procedures as Triton kernels and call them from PyTorch code that implements the rest of the \ours algorithm.
As in FlashAttention \citep{dao2022flashattention}, to maximize GPU utilization, we use tiling to reduce slow high bandwidth memory (HBM) reads and writes and avoid materializing the kernel matrices used in \cref{algo:khd}. 
We also parallelize across batch, head, and row blocks of input query tensors as in FlashAttention 2 \citep{dao2024flashattention}.
When processing long contexts, we leverage two additional optimizations.
Since all query-key-value tuples are observed at once and available on HBM, we parallelize all \halve calls on same-sized groups in \cref{algo:compress2} and avoid  unnecessary data copies by using double dereferencing to index into the key and value vectors.

\cref{fig:unmasked-prefill} demonstrates the substantial speed-ups of our Triton Thinformer implementation over the \verb|torch.compile| PyTorch implementation of \citet{carrell2025low} when approximating unmasked attention over long contexts with cache size $\nout=256$.
At context length 32k, Thinformer (Original) is already $15\times$ faster than the highly-optimized FlashAttention 2 kernel for exact attention, but Thinformer (Triton) is $32\times$ faster, highlighting the value of our I/O-aware optimizations.

\para{Accelerating long-context prefill}\label{sec:prefill}
In language modeling, the input context must be fully processed before generation can begin---a stage known as prefill. Prefill cost grows quadratically with context length, making it a key bottleneck for long sequences.
Leveraging our I/O-aware Triton implementation, we evaluate the speed-up of \ours with $\mbar=2\floor{\errsymb/2}-2$ over FlashAttention 2 with causal masking across a range of long context lengths, following the experimental setup of \citet[Fig.~4(a)]{han2024hyperattention}.
\cref{fig:masked-prefill} shows the speed-up of \ours over FlashAttention 2 growing with sequence length $n$, reaching $61\times$ speed-up for $\errsymb=9$ and $30\times$ speed-up for $\errsymb=10$ at $n=512$k tokens.
Notably, the state-of-the-art context approximation method HyperAttention \citep{han2024hyperattention} runs out of memory at the longest sequence length of 512k on an A6000 GPU, while \ours does not.

To measure the trade-off between runtime and quality, \cref{fig:masked-prefill-32k} recreates the benchmark of \citet[Fig.~3]{han2024hyperattention} by assessing both runtime and perplexity when replacing exact attention in the last $k$ layers of the ChatGLM2-6B-32K language model \citep{glm2024chatglm}. Here, all methods preserve the first and last $32$ context tokens and compress the intermediate tokens. 
We find that both settings of \ours offer a better speed-up-vs.-perplexity trade-off than HyperAttention,  that \ours ($\errsymb=10$) uniformly improves both the perplexity and speed-up of HyperAttention, and that, even with all layers replaced, all methods evaluated preserve the perplexity of exact attention up to a factor of $1.06$. 
Moreover, these results are insensitive to the choice of $\delta$ with near-identical performance provided by $\delta=0.01$ (\cref{fig:masked-prefill-32k-small-delta}). 
Finally, in \cref{fig:maxnorm}, we directly compare the output of \ours $\attouthat_n$ with the output of exact attention $\attout_n$ and verify that the scaled attention approximation error $\infnorm{\attouthat_n - \attout_n}/\sqrt{\log n \log\log n}$ decays with increasing cache size $\nout=2^{\errsymb}$ and remains bounded for all $n$, as predicted by \cref{att-err}.

\para{Accelerating KV cache compression}\label{sec:kvcache}
\begin{figure}[t!]
\centering
        \includegraphics[width=\textwidth]{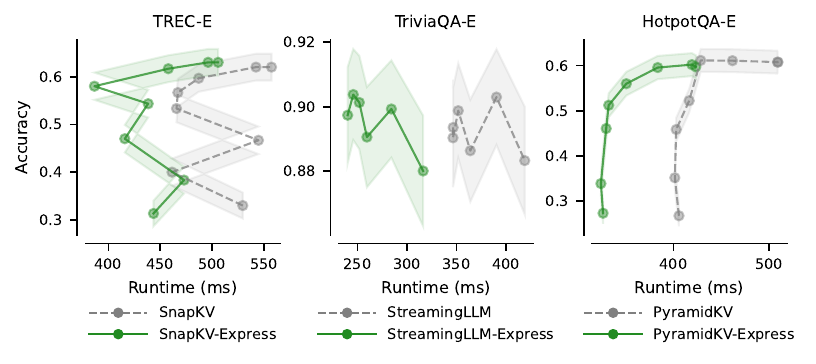}
    \caption{\tbf{Accelerating KV cache compression.} Across leading KV cache compression methods and LongBench-E tasks, \ours substantially reduces runtime while preserving quality. Error bands display $\pm 1$ standard error across all  benchmark questions. %
    See \cref{sec:kvcache} for more details.
    }
    \label{fig:kv-cache-compression}
\end{figure}
In transformer-based autoregressive models, the keys and values from prior tokens are stored in a KV cache to avoid recomputation during generation. As the context length grows, this cache quickly becomes a memory bottleneck.
Leading KV cache compression methods address the memory bottleneck by preserving only a smaller set of representative key-value pairs \citep{li2024snapkv,xiao2024streamingllm,cai2025pyramidkv}, but exactly computing all input keys and values prior to compression still incurs a quadratic attention cost.

To accelerate KV cache compression, one could instead cheaply generate the input keys and values using \ours and then compress the resulting cache using any high-quality compressor. 
Here, we instantiate this general recipe using three leading cache compression methods---SnapKV \citep{li2024snapkv}, StreamingLLM \citep{xiao2024streamingllm}, and PyramidKV \citep{cai2025pyramidkv}---applied to three benchmark long-context language understanding tasks \citep{bai2024longbench}: TREC-E \citep{li2002learning} (question classification across 50 fine-grained classes), TriviaQA-E \citep{joshi2017triviaqa} (few-shot question answering), and HotpotQA-E \citep{yang2018hotpotqa} (multi-hop question answering). 
In \cref{fig:kv-cache-compression}, we vary KV cache size and compare the accuracy-runtime trade-off curves induced by using exact attention versus \ours ($\errsymb=9$) in all layers of the Llama 3.1 8B Instruct model  \citep{grattafiori2024llama}.
Remarkably, in each case, \ours preserves KV cache compression quality while substantially reducing total attention runtime.

\begin{figure}[t!]
    \centering
    \includegraphics[width=0.49\textwidth]{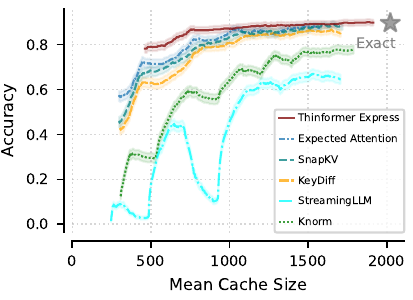}\hfill
    \includegraphics[width=0.49\textwidth]{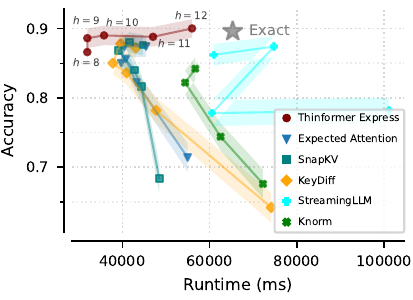}
    \caption{\textbf{Memory- and compute-efficient long-form decoding.} On competition-level MATH-500  problems that benefit from step-by-step reasoning,  \ours improves both the cache size-accuracy trade-off (left) and time-accuracy trade-off (right) of long-form decoding. Error bands display $\pm 1$ standard error across all 500 benchmark questions. See \cref{sec:mem-long,sec:decoding} for more details.
    }
    \label{fig:long-form}
\end{figure}
\para{Memory-efficient long-form decoding}\label{sec:mem-long}
A second memory bottleneck in language modeling arises from long-form decoding as each newly generated token contributes an additional key and value vectors to the ever-growing KV cache. 
For example, on competition-level MATH-500 problems that benefit from step-by-step reasoning \citep{hendrycks2021measuring,lightman2023let}, the DeepSeek-R1-Distill-Llama-8B model \citep{guo2025deepseek} generates over 4k tokens on average even though the input questions are only $75$ tokens long on average. 
In such scenarios, \ours can be used to maintain a substantially smaller, dynamically-updated KV cache that preserves the quality of exact attention. 

To benchmark this streaming cache compression quality, we compare \ours with five leading KV cache compression methods: 
StreamingLLM~\citep{xiao2024streamingllm}, SnapKV~\citep{li2024snapkv}, ExpectedAttention~\citep{devoto2025expectedattention}, KeyDiff~\citep{parkkeydiff}, and Knorm~\citep{devoto2024simple}. StreamingLLM is a natively streaming cache compression method, while the other four use the batch-to-online conversion scheme of \citet{devoto2025expectedattention}.\footnote{Knorm also admits a streaming version, but it underperformed the batch-to-online conversion in our experiments.}
For each method, \cref{fig:long-form} (left) displays the cache size-vs.-accuracy trade-off curves for DeepSeek-R1-Distill-Llama-8B on MATH-500. Notably, \ours achieves the highest accuracy at nearly every cache size and 
retains 
99\%
of the accuracy of exact attention using only 
61\% 
of the cache elements.
\para{Accelerating long-form decoding}\label{sec:decoding}
When runtime is the primary bottleneck, \ours can also be used to alleviate the computational burden of quadratic-time long-form decoding by reducing runtime per token while accurately preserving the exact attention output. 
To evaluate its potential for accelerating long-form decoding, we benchmark the runtime and accuracy of \ours applied to DeepSeek-R1-Distill-Llama-8B on competition-level MATH-500 reasoning problems. 
In \cref{fig:long-form} (right), we find that the \ours trade-off curve strongly dominates those of the five leading alternatives---StreamingLLM, SnapKV, ExpectedAttention, KeyDiff, and Knorm---and retains 
99\% 
of the accuracy of exact attention with only 
55\% 
of the computation time. 
In line with our theory, we further find that the runtime overhead of maintaining the \ours cache is quite low when compared with the query cost of attending over the cache. 
For instance, across all benchmark questions, \ours ($\errsymb=10$) expends 35400 ms on average calling \wattn, compared with only 390 ms calling \khexpress.update().
\section{Reflections and Conclusions}\label{sec:conclusion}
We introduced \express, a meta-procedure for causal masking, and \ours, a new causal attention approximation with per-token accuracy guarantees, constant memory, sub-quadratic query time, low compression overhead, and an efficient I/O-aware Triton GPU kernel.
In theory, we provide sharper and more resource-efficient guarantees than prior work, and, in practice, we demonstrate improved performance across four bottlenecks in the language modeling pipeline: long-context prefill, KV cache compression, long-form memory-constrained  decoding, and long-form compute-constrained decoding. 
However, this work is not without its limitations. 
For example, our experiments only cover English language modeling, Chinese language modeling,   %
and mathematical reasoning and have not yet evaluated the strengths and weaknesses of \express for other languages or domains. 
the \express framework is general and its guarantees are kernel-agnostic, so extensions to other modalities are possible with appropriate kernel choices.
While we instantiate \express with kernel halving, the meta-procedure accepts any sub-Gaussian halving algorithm, and our modular analysis allows any such algorithm to be substituted directly. 
Moreover, our Triton implementation does not yet exploit warp specialization, the Tensor Memory Accelerator supported in newer GPU architectures, persistent kernels, or FP8 support as in FlashAttention 2 and 3 \citep{dao2024flashattention,shah2024flashattention}; each is directly applicable and expected to yield further practical speed-ups.

Finally, while we have focused on combining \express with the state-of-the-art Thinformer approximation, \express is compatible with any halving algorithm, even those that have no known guarantees but perform well in practice. 
For example, the authors of BalanceKV~\citep{han2025streaming} did not release an implementation of their causal, streaming attention algorithm but did release an implementation of their unmasked halving algorithm.  
Directly applying \express to BalanceKV halving yields a causal, streaming attention approximation that strongly outperforms the leading alternatives of \cref{sec:experiments} for both memory-constrained and compute-constrained decoding (\cref{fig:long-form-balancekv}).

\begin{ack}
The authors thank Amir Zandieh and Insu Han for valuable advice concerning the long-context prefill experiments and Paula Cordero Encinar for stress-testing and debugging the codebase. 
AG gratefully acknowledges support from the AI Research Institutes program of the National Science Foundation and Intel Corporation under NSF Award DMR-2433348.
\end{ack}

\bibliographystyle{plainnat} %
\bibliography{refs}

\begin{thebibliography}{38}
\providecommand{\natexlab}[1]{#1}
\providecommand{\url}[1]{\texttt{#1}}
\expandafter\ifx\csname urlstyle\endcsname\relax
  \providecommand{\doi}[1]{doi: #1}\else
  \providecommand{\doi}{doi: \begingroup \urlstyle{rm}\Url}\fi

\bibitem[Alweiss et~al.(2021)Alweiss, Liu, and Sawhney]{alweiss2020discrepancyminimizationselfbalancingwalk}
Ryan Alweiss, Yang~P Liu, and Mehtaab Sawhney.
\newblock Discrepancy minimization via a self-balancing walk.
\newblock In \emph{Proceedings of the 53rd Annual ACM SIGACT Symposium on Theory of Computing}, pages 14--20, 2021.

\bibitem[Aronszajn(1950)]{aronszajn1950reproducing}
N.~Aronszajn.
\newblock Theory of reproducing kernels.
\newblock \emph{Transactions of the American Mathematical Society}, 68\penalty0 (3):\penalty0 337--404, 1950.
\newblock URL \url{http://dx.doi.org/10.2307/1990404}.

\bibitem[Bai et~al.(2024)Bai, Lv, Zhang, Lyu, Tang, Huang, Du, Liu, Zeng, Hou, Dong, Tang, and Li]{bai2024longbench}
Yushi Bai, Xin Lv, Jiajie Zhang, Hongchang Lyu, Jiankai Tang, Zhidian Huang, Zhengxiao Du, Xiao Liu, Aohan Zeng, Lei Hou, Yuxiao Dong, Jie Tang, and Juanzi Li.
\newblock Longbench: A bilingual, multitask benchmark for long context understanding.
\newblock In \emph{Proceedings of the 62nd annual meeting of the association for computational linguistics (volume 1: Long papers)}, pages 3119--3137, 2024.

\bibitem[Cai et~al.(2025)Cai, Zhang, Gao, Liu, Li, Liu, Lu, Xiong, Dong, Hu, et~al.]{cai2025pyramidkv}
Zefan Cai, Yichi Zhang, Bofei Gao, Yuliang Liu, Yucheng Li, Tianyu Liu, Keming Lu, Wayne Xiong, Yue Dong, Junjie Hu, et~al.
\newblock Pyramidkv: Dynamic kv cache compression based on pyramidal information funneling.
\newblock In \emph{Second Conference on Language Modeling}, 2025.

\bibitem[Carrell et~al.(2025)Carrell, Gong, Shetty, Dwivedi, and Mackey]{carrell2025low}
Annabelle~Michael Carrell, Albert Gong, Abhishek Shetty, Raaz Dwivedi, and Lester Mackey.
\newblock Low-rank thinning.
\newblock In \emph{International Conference on Machine Learning}, pages 6811--6848. PMLR, 2025.

\bibitem[{Chameleon Team}(2024)]{team2024chameleon}
{Chameleon Team}.
\newblock Chameleon: Mixed-modal early-fusion foundation models.
\newblock \emph{arXiv preprint arXiv:2405.09818}, 2024.

\bibitem[Chazelle and Matousek(1996)]{chazelle1996linear}
Bernard Chazelle and Jiri Matousek.
\newblock On linear-time deterministic algorithms for optimization problems in fixed dimension.
\newblock \emph{Journal of Algorithms}, 21\penalty0 (3):\penalty0 579--597, 1996.

\bibitem[Dao(2024)]{dao2024flashattention}
Tri Dao.
\newblock Flashattention-2: Faster attention with better parallelism and work partitioning.
\newblock In \emph{The Twelfth International Conference on Learning Representations}, 2024.
\newblock URL \url{https://openreview.net/forum?id=mZn2Xyh9Ec}.

\bibitem[Dao et~al.(2022)Dao, Fu, Ermon, Rudra, and R{\'e}]{dao2022flashattention}
Tri Dao, Dan Fu, Stefano Ermon, Atri Rudra, and Christopher R{\'e}.
\newblock Flashattention: Fast and memory-efficient exact attention with io-awareness.
\newblock \emph{Advances in {N}eural {I}nformation {P}rocessing {S}ystems}, 35:\penalty0 16344--16359, 2022.

\bibitem[Devoto et~al.(2024)Devoto, Zhao, Scardapane, and Minervini]{devoto2024simple}
Alessio Devoto, Yu~Zhao, Simone Scardapane, and Pasquale Minervini.
\newblock A simple and effective l\_2 norm-based strategy for kv cache compression.
\newblock In \emph{Proceedings of the 2024 Conference on Empirical Methods in Natural Language Processing}, pages 18476--18499, 2024.

\bibitem[Devoto et~al.(2025)Devoto, Jeblick, and J{\'e}gou]{devoto2025expectedattention}
Alessio Devoto, Maximilian Jeblick, and Simon J{\'e}gou.
\newblock Expected attention: Kv cache compression by estimating attention from future queries distribution.
\newblock \emph{arXiv preprint arXiv:2510.00636}, 2025.

\bibitem[Dwivedi and Mackey(2022)]{dwivedi2021generalized}
Raaz Dwivedi and Lester Mackey.
\newblock Generalized kernel thinning.
\newblock In \emph{International Conference on Learning Representations}, 2022.

\bibitem[Dwivedi and Mackey(2024)]{dwivedi2024kernel}
Raaz Dwivedi and Lester Mackey.
\newblock Kernel thinning.
\newblock \emph{Journal of Machine Learning Research}, 25\penalty0 (152):\penalty0 1--77, 2024.

\bibitem[Glm et~al.(2024)Glm, Zeng, Xu, Wang, Zhang, Yin, Zhang, Rojas, Feng, Zhao, et~al.]{glm2024chatglm}
Team Glm, Aohan Zeng, Bin Xu, Bowen Wang, Chenhui Zhang, Da~Yin, Dan Zhang, Diego Rojas, Guanyu Feng, Hanlin Zhao, et~al.
\newblock Chatglm: A family of large language models from glm-130b to glm-4 all tools.
\newblock \emph{arXiv preprint arXiv:2406.12793}, 2024.

\bibitem[Grattafiori et~al.(2024)Grattafiori, Dubey, Jauhri, Pandey, Kadian, Al-Dahle, Letman, Mathur, Schelten, Vaughan, et~al.]{grattafiori2024llama}
Aaron Grattafiori, Abhimanyu Dubey, Abhinav Jauhri, Abhinav Pandey, Abhishek Kadian, Ahmad Al-Dahle, Aiesha Letman, Akhil Mathur, Alan Schelten, Alex Vaughan, et~al.
\newblock The llama 3 herd of models.
\newblock \emph{arXiv preprint arXiv:2407.21783}, 2024.

\bibitem[Guo et~al.(2025)Guo, Yang, Zhang, Song, Wang, Zhu, Xu, Zhang, Ma, Bi, et~al.]{guo2025deepseek}
Daya Guo, Dejian Yang, Haowei Zhang, Junxiao Song, Peiyi Wang, Qihao Zhu, Runxin Xu, Ruoyu Zhang, Shirong Ma, Xiao Bi, et~al.
\newblock Deepseek-r1 incentivizes reasoning in llms through reinforcement learning.
\newblock \emph{Nature}, 645\penalty0 (8081):\penalty0 633--638, 2025.

\bibitem[Han et~al.(2024)Han, Jayaram, Karbasi, Mirrokni, Woodruff, and Zandieh]{han2024hyperattention}
Insu Han, Rajesh Jayaram, Amin Karbasi, Vahab Mirrokni, David Woodruff, and Amir Zandieh.
\newblock Hyperattention: Long-context attention in near-linear time.
\newblock In \emph{The Twelfth International Conference on Learning Representations}, 2024.
\newblock URL \url{https://openreview.net/forum?id=Eh0Od2BJIM}.

\bibitem[Hendrycks et~al.(2021)Hendrycks, Burns, Kadavath, Arora, Basart, Tang, Song, and Steinhardt]{hendrycks2021measuring}
Dan Hendrycks, Collin Burns, Saurav Kadavath, Akul Arora, Steven Basart, Eric Tang, Dawn Song, and Jacob Steinhardt.
\newblock Measuring mathematical problem solving with the {MATH} dataset.
\newblock In \emph{Thirty-fifth Conference on Neural Information Processing Systems Datasets and Benchmarks Track (Round 2)}, 2021.
\newblock URL \url{https://openreview.net/forum?id=7Bywt2mQsCe}.

\bibitem[Hoeffding(1963)]{409cf137-dbb5-3eb1-8cfe-0743c3dc925f}
Wassily Hoeffding.
\newblock Probability inequalities for sums of bounded random variables.
\newblock \emph{Journal of the American Statistical Association}, 58\penalty0 (301):\penalty0 13--30, 1963.
\newblock ISSN 01621459, 1537274X.
\newblock URL \url{http://www.jstor.org/stable/2282952}.

\bibitem[Joshi et~al.(2017)Joshi, Choi, Weld, and Zettlemoyer]{joshi2017triviaqa}
Mandar Joshi, Eunsol Choi, Daniel~S Weld, and Luke Zettlemoyer.
\newblock Triviaqa: A large scale distantly supervised challenge dataset for reading comprehension.
\newblock In \emph{Proceedings of the 55th Annual Meeting of the Association for Computational Linguistics (Volume 1: Long Papers)}, pages 1601--1611, 2017.

\bibitem[Kochetkova et~al.(2025)Kochetkova, Sheth, Han, Zandieh, and Kapralov]{han2025streaming}
Ekaterina Kochetkova, Kshiteej Sheth, Insu Han, Amir Zandieh, and Michael Kapralov.
\newblock Streaming attention approximation via discrepancy theory.
\newblock In \emph{Advances in Neural Information Processing Systems}, 2025.

\bibitem[Li and Roth(2002)]{li2002learning}
Xin Li and Dan Roth.
\newblock Learning question classifiers.
\newblock In \emph{COLING 2002: The 19th International Conference on Computational Linguistics}, 2002.

\bibitem[Li et~al.(2024)Li, Huang, Yang, Venkitesh, Locatelli, Ye, Cai, Lewis, and Chen]{li2024snapkv}
Yuhong Li, Yingbing Huang, Bowen Yang, Bharat Venkitesh, Acyr Locatelli, Hanchen Ye, Tianle Cai, Patrick Lewis, and Deming Chen.
\newblock Snapkv: Llm knows what you are looking for before generation.
\newblock In \emph{Advances in Neural Information Processing Systems}, volume~37, 2024.

\bibitem[Liberty et~al.(2026)Liberty, Andoni, and Kleiner]{liberty2026nearly}
Edo Liberty, Alexandr Andoni, and Eldar Kleiner.
\newblock Nearly optimal attention coresets.
\newblock \emph{arXiv preprint arXiv:2605.05602}, 2026.

\bibitem[Lightman et~al.(2023)Lightman, Kosaraju, Burda, Edwards, Baker, Lee, Leike, Schulman, Sutskever, and Cobbe]{lightman2023let}
Hunter Lightman, Vineet Kosaraju, Yuri Burda, Harrison Edwards, Bowen Baker, Teddy Lee, Jan Leike, John Schulman, Ilya Sutskever, and Karl Cobbe.
\newblock Let's verify step by step.
\newblock In \emph{The twelfth international conference on learning representations}, 2023.

\bibitem[Matousek(1995)]{matousek1995approximations}
Jiri Matousek.
\newblock Approximations and optimal geometric divide-and-conquer.
\newblock \emph{Journal of Computer and System Sciences}, 50\penalty0 (2):\penalty0 203--208, 1995.

\bibitem[Park et~al.(2025)Park, Jones, Morse, Goel, Lee, and Lott]{parkkeydiff}
Junyoung Park, Dalton Jones, Matthew~J Morse, Raghavv Goel, Mingu Lee, and Christopher Lott.
\newblock Keydiff: Key similarity-based kv cache eviction for long-context llm inference in resource-constrained environments.
\newblock In \emph{The Thirty-ninth Annual Conference on Neural Information Processing Systems}, 2025.

\bibitem[Phillips(2008)]{phillips2008algorithms}
Jeff~M Phillips.
\newblock Algorithms for $\varepsilon$-approximations of terrains.
\newblock In \emph{International Colloquium on Automata, Languages, and Programming}, pages 447--458. Springer, 2008.

\bibitem[Schr{\"o}der and Mackey(2026)]{schroder2026wildcat}
Tobias Schr{\"o}der and Lester Mackey.
\newblock Wildcat: Near-linear attention in theory and practice.
\newblock In \emph{Forty-third International Conference on Machine Learning}, 2026.
\newblock URL \url{https://openreview.net/forum?id=lfqyLp4hZm}.

\bibitem[Shah et~al.(2024)Shah, Bikshandi, Zhang, Thakkar, Ramani, and Dao]{shah2024flashattention}
Jay Shah, Ganesh Bikshandi, Ying Zhang, Vijay Thakkar, Pradeep Ramani, and Tri Dao.
\newblock Flashattention-3: Fast and accurate attention with asynchrony and low-precision.
\newblock \emph{Advances in Neural Information Processing Systems}, 37:\penalty0 68658--68685, 2024.

\bibitem[Shetty et~al.(2022)Shetty, Dwivedi, and Mackey]{shetty2022distribution}
Abhishek Shetty, Raaz Dwivedi, and Lester Mackey.
\newblock Distribution compression in near-linear time.
\newblock In \emph{International Conference on Learning Representations}, 2022.

\bibitem[Steinwart and Christmann(2008)]{Steinwart2008SupportVM}
Ingo Steinwart and Andreas Christmann.
\newblock \emph{Support Vector Machines}.
\newblock Information Science and Statistics. Springer, New York, NY, 2008.
\newblock \doi{10.1007/978-0-387-77242-4}.

\bibitem[Tillet et~al.(2019)Tillet, Kung, and Cox]{tillet2019triton}
Philippe Tillet, Hsiang-Tsung Kung, and David Cox.
\newblock Triton: an intermediate language and compiler for tiled neural network computations.
\newblock In \emph{Proceedings of the 3rd ACM SIGPLAN International Workshop on Machine Learning and Programming Languages}, pages 10--19, 2019.

\bibitem[Vaswani et~al.(2017)Vaswani, Shazeer, Parmar, Uszkoreit, Jones, Gomez, Kaiser, and Polosukhin]{vaswani2017attention}
Ashish Vaswani, Noam Shazeer, Niki Parmar, Jakob Uszkoreit, Llion Jones, Aidan~N. Gomez, \L{}ukasz Kaiser, and Illia Polosukhin.
\newblock Attention is all you need.
\newblock In \emph{Proceedings of the 31st International Conference on Neural Information Processing Systems}, NIPS'17, page 6000–6010, Red Hook, NY, USA, 2017. Curran Associates Inc.
\newblock ISBN 9781510860964.

\bibitem[Veli{\v{c}}kovi{\'c} et~al.(2024)Veli{\v{c}}kovi{\'c}, Perivolaropoulos, Barbero, and Pascanu]{velivckovic2024softmax}
Petar Veli{\v{c}}kovi{\'c}, Christos Perivolaropoulos, Federico Barbero, and Razvan Pascanu.
\newblock Softmax is not enough (for sharp size generalisation).
\newblock \emph{arXiv preprint arXiv:2410.01104}, 2024.

\bibitem[Xiao et~al.(2024)Xiao, Tian, Chen, Han, and Lewis]{xiao2024streamingllm}
Guangxuan Xiao, Yuandong Tian, Beidi Chen, Song Han, and Mike Lewis.
\newblock Efficient streaming language models with attention sinks.
\newblock In \emph{International Conference on Learning Representations}, 2024.
\newblock URL \url{https://openreview.net/forum?id=NG7sS51zVF}.

\bibitem[Yang et~al.(2018)Yang, Qi, Zhang, Bengio, Cohen, Salakhutdinov, and Manning]{yang2018hotpotqa}
Zhilin Yang, Peng Qi, Saizheng Zhang, Yoshua Bengio, William Cohen, Ruslan Salakhutdinov, and Christopher~D Manning.
\newblock Hotpotqa: A dataset for diverse, explainable multi-hop question answering.
\newblock In \emph{Proceedings of the 2018 conference on empirical methods in natural language processing}, pages 2369--2380, 2018.

\bibitem[Zandieh et~al.(2023)Zandieh, Han, Daliri, and Karbasi]{zandieh2023kdeformer}
Amir Zandieh, Insu Han, Majid Daliri, and Amin Karbasi.
\newblock Kdeformer: Accelerating transformers via kernel density estimation.
\newblock In \emph{International Conference on Machine Learning}, pages 40605--40623. PMLR, 2023.

\end{thebibliography}

\newpage
\appendix

\numberwithin{mylemma}{section}
\numberwithin{myproposition}{section}
\numberwithin{mydefinition}{section}
\numberwithin{theorem}{section}
\numberwithin{mycorollary}{section}
\numberwithin{myremark}{section}
\numberwithin{condition}{section}
\numberwithin{problem}{section}
\numberwithin{myassumption}{section}
\numberwithin{property}{section}
\numberwithin{example}{section}
\numberwithin{algocf}{section}
\numberwithin{figure}{section}
\numberwithin{table}{section}

\section{Kernel Halving}\label{app:appendix-algorithms}\label{app:subg_thinning_algorithms}

For completeness, we reproduce the \khd halving algorithm of \citet[Alg.~B.1]{carrell2025low}.
\begin{algorithm2e}[ht!]
\caption{\khd: Kernel halving with simplified swapping thresholds and failure probability $\delta/2$}
\label{algo:khd}
\small{
  \KwIn{point sequence $\xin=(\x_i)_{i = 1}^{\nin}$ with even $\nin$, kernel $\kernel$}
  \BlankLine
  {$\coreset[1], \coreset[2] \gets \braces{}$};\quad $\multiplier_{\max,0} \gets 0$\quad // Initialize empty coresets and max function norm \\
  \For{$i=1, 2, \ldots, \nin/2$}
    {%
    // Construct kernel difference function using next two points \\
    $(\x, \x') \gets (\x_{2i-1}, \x_{2i})$;\quad
    $f_i \gets \kernel(\x_{2i-1}, \cdot)-\kernel(\x_{2i}, \cdot)$ \\
	 \BlankLine
     // Compute swapping threshold $\thresh_i$ \\ %
     $\multiplier_i^2 \!=\! \norm{f_i}_{\kernel}^2
      \!=\! \kernel(\x,\x)\!+\!\kernel(\x',\x')\!-\!2\kernel(\x,\x')$;\quad $\multiplier_{\max,i} = \max(\multiplier_i, \multiplier_{\max,i-1})$ \\
     $ \thresh_i \gets \multiplier_i \multiplier_{\max,i}(\half + \log(2\nin/\delta))$
     \ \\
   \BlankLine
    // Compute the balancing score for the next pair \\
    $\alpha_i\gets  \sum_{j=1}^{2i-2}(\kernel
	 (\x_j, \x)-\kernel(\x_j,\x'))
	 - 2\sum_{\z\in\coreset[1]}(\kernel(\z, \x)-\kernel(\z,\x'))$ \\
    \BlankLine
			 // Assign one point to each coreset after probabilistic swapping \\[2pt]
		     $(\x, \x') \gets (\x', \x)$ \qtext{\textit{with probability}} $\min(1, \half (1-\frac{\alpha_i}{\thresh_i})_+)$ \\
           $\coreset[1]\texttt{.append}(\x);
		        \ \ \  \coreset[2]\texttt{.append}(\x')$
  }
  \KwRet{\textup{$\xout\defeq\coreset[1]$, coreset of size} $\nout=\nin/2$}{}
  }
\end{algorithm2e}

\section{\pcref{express-guarantees}}
\label{proof:express-guarantees}

\paragraph{Proof of runtime claim~\cref{eq:express-runtime}.}
For $n < 4\nout$, \express has not yet called \halve, and \compresstwo with $m=0$ stores every point without halving, so $r_{\expresstag}(n) = 0$. 

Next suppose $n \geq 4\nout$. 
Fix any \express round $m \in \braces{0,2,4,\ldots}$, and let $m_\star \defeq m \wedge \mbar$.
If $m \leq \mbar$, then $S.\keep(\x)$ is always true, so \subsample does not filter any points. 
If $m > \mbar$, \subsample (\cref{algo:subsample,algo:subsample-recursive}) partitions the level-$m$ block into $2^{\mbar}\nout$ strata, each of size $2^{m-\mbar}$, and draws one point from each stratum.
Since the sampler draws one index per output stratum, the total cost to subsample $2^{\mbar}\nout$ points from $2^{m}\nout$ is $2^{\mbar}\nout$ uniform sampling operations. Since $r_{\compresstag}(m_\star)$ is the cost of $\compresstwo(m_\star,2^{m_\star}\nout,\halve_m)$, 
$r_{\halvetag}(4 \nout)
    +
r_{\halvetag}(2 \nout)$
is the cost of the \halve phase, 
and three Thin phases (each thinning $2^m\nout$ points) occur in \express prior to the \halve phase,  $r_{\expresstag}$ satisfies the  recursion $r_{\expresstag}(\nout) = 0$ and 
\begin{talign}\label{eq:express-runtime-1}
r_{\expresstag}(2^{m+2}\nout )
    &=
r_{\expresstag}(2^{m}\nout )
    +
3\indicator\braces{m > \mbar} 2^{\mbar}\nout
    +
3r_{\compresstag}(m_\star)
    +
r_{\halvetag}(4 \nout)
    +
r_{\halvetag}(2 \nout).
\end{talign}
Unrolling the recursion \cref{eq:express-runtime-1} yields the runtime 
\begin{talign}\label{eq:express-runtime-detailed}
\!\!\!r_{\expresstag}(n) 
    &=
\indic{n\geq 4\nout}\sum_{k=0}^{K_n-1}\! \big(
r_{\halvetag}(4 \nout) + r_{\halvetag}(2 \nout)
    +
3r_{\compresstag}\parenth{(2k) \wedge \mbar}
    +
3 \indic{k>\frac{\mbar}{2}} 2^{\mbar}\nout\big) \\
    &\leq 
\indic{n\geq 4\nout}K_n\big(
r_{\halvetag}(4 \nout) + r_{\halvetag}(2 \nout)
    +
3r_{\compresstag}\parenth{\mbar}
    +
3 \cdot 2^{\mbar}\nout\big)
\end{talign}
for $K_n \defeq \ceil{\log_4(n/\nout)}$ whenever $n=2^{2K_n}\nout$. 
Finally, for intermediate values of $n$, we have $r_{\expresstag}(n) \leq r_{\expresstag}(2^{2K_n}\nout)$ since $r_{\expresstag}$ is non-decreasing.

\paragraph{Proof of space claim.}
For $n\leq \nout$, \express stores all input points and the space claim follows immediately.
Now suppose $n>\nout$.
A convenient way of computing the weighted cache size is to consider numbers in base 4.
At the end of an update call in \cref{algo:express}, let $m$ be the current thinning factor, $q \defeq m\wedge\mbar$, and $b$ be the completed Thin phases since the last \halve phase. 
After an update returns, three completed Thin phases trigger a \halve phase, so $b\leq 2$.
At this point, $E$ contains $\nout$ points (from the Exact or most recent \halve phase) plus $b\nout$ points from the concatenation event at the end of each Thin phase (Line~\ref{line:express-store-compress2}), so $\abss{E}=(b+1)\nout$.

Let $u=\sum_{i=0}^{q-1} |\cset_i|$ be the number of points retained in $C.\wtdcache()$ and $u_{\expresstag} = u + |E|$. 
Define $x,a$ to be the integers satisfying $u = \nout 2^{2-q}x + a$ and $0 \leq a < \nout 2^{2-q}$.
Let $x = \sum_{i=1}^{q-1} 4^{i-1} d_i$, where $d_i \in \{0,1,2,3\}$ for all $i$.
Then the current \compresstwo cache satisfies $\abss{\cset_0}=a$ and $\abss{\cset_i}=\nout 2^{i-q}d_i$ for $i=1,\ldots,q-1$.
Thus 
\begin{talign}\label{express_cache_size}
u_{\expresstag}
    =
(b+1)\nout + a + \sum_{i=1}^{q-1}\nout 2^{i-q} d_i.
\end{talign}
If $q=0$, then $u\leq\nout$ points, so $u_{\expresstag} \leq 4\nout$ because $b\leq 2$.
If $q\geq 1$, then the worst case has $b=2$, $a$ arbitrarily close to $\nout 2^{2-q}$, and $d_i=3$ for all $i$.
In that case, 
\begin{talign}
u
    =
\nout 2^{2-q} + 3\nout \sum_{i=1}^{q-1} 2^{i-q}
    =
\parenth{3-2^{1-q}}\nout.
\end{talign}
Adding $|E|$ gives $(6-2^{1-q})\nout < 6\nout$.
Hence, for all $n\in\naturals$,
\begin{talign}\label{max_cache_size}
u_{\expresstag}
    =
(4+b-2^{1-q})\nout
    <
6\nout.
\end{talign}

Now suppose \halve with input size $\nin$ uses at most $s_{\halvetag}(\nin)$ space for nondecreasing $s_{\halvetag}$.
Since \express calls \halve sequentially, it can reuse the same memory for each \halve call.
The \compresstwo calls at level $q$ invoke \halve on inputs of size $\nout 2^{2-q+i}$ for $i=0,\ldots,q-1$, which are at most $2\nout$, while the two warm-start \halve calls in \cref{algo:express} have input sizes $4\nout$ and $2\nout$.
Thus the temporary halving workspace is bounded by $s_{\halvetag}(4\nout)$, and
\begin{talign}\label{express_storage_cost}\label{eq:streaming_express_storage_cost}
s_{\expresstag}(n) \leq
\begin{cases}
    nd & \text{if } n \leq \nout, \\
    u_{\expresstag}d + s_{\halvetag}(4\nout) & \text{otherwise}.
\end{cases}
\end{talign}
Combining this display with \cref{max_cache_size} yields our claim.

\section{\pcref{rmk:express-polynomial-runtime}}
\label{proof:express-polynomial-runtime}

Suppose $r_{\halvetag}(\nin)\le\rho\nin^\tau$ for all $\nin\in[4\nout]$. 
The claim is immediate for $n<4\nout$, so suppose $n\geq4\nout$ with $K_n \defeq \ceil{\log_4(n/\nout)}$ and $L_n\defeq K_n\wedge(\floor{\mbar/2}+1)$.
\cref{express-guarantees} and the constraint $2^{\mbar}\leq 2\nout$ imply that
\begin{talign}
r_{\expresstag}(n) 
    &\leq
\rho K_n(4^\tau+2^\tau) \nout^{\tau}
    +
6 K_n\nout^2
    +
3K_nr_{\compresstag}\parenth{\mbar}. \label{eq:ren}
\end{talign}

It remains to bound $r_{\compresstag}$.
If $\tau=2$, our claim follows from \cref{eq:ren} and the estimate
\begin{talign}
r_{\compresstag}(\mbar)
    &\le
\sum_{j=0}^{\mbar-1}4^j\rho\parenth{\nout 2^{1-j}}^2
    =
4\rho \mbar\nout^2
    \le
4\rho \log_2(2\nout)\nout^2.
\end{talign}
Next suppose $\tau\neq 2$. 
Then, since $2^{\mbar}\leq 2\nout$,
\begin{talign}
r_{\compresstag}(\mbar)
    &\leq
\sum_{j=0}^{\mbar-1}4^j\rho\parenth{\nout 2^{1-j}}^\tau
    =
\rho(2\nout)^\tau 
\sum_{j=0}^{\mbar-1}2^{(2-\tau)j}
    =
\rho \frac{(2\nout)^\tau|1-2^{(2-\tau)\mbar}|}{|1-2^{2-\tau}|}
\\
    &\leq \label{eq:rcm-unequal}
\rho \frac{(2\nout)^\tau|1-(2\nout)^{2-\tau}|}{|1-2^{2-\tau}|}
    =
\rho 2^\tau\frac{|(2\nout)^\tau-(2\nout)^{2}|}{|2^\tau-2^{2}|}.
\end{talign}
Our $\tau\neq 2$ claim now follows immediately from the estimates \cref{eq:ren,,eq:rcm-unequal}.

\section{\pcref{express-quality}}
\label{proof:express-quality}

We begin by recalling the definition of sub-Gaussianity on an event $\event$ \citep[Def.~5]{shetty2022distribution}.
\begin{definition}[\tbf{Sub-Gaussian on an event}]\label{def:subg}
We say a random variable $X$ is $\subg$-sub-Gaussian on an event $\event$ if $\subg \ge 0$ and 
\begin{talign}
\Esub{\event}\brackets{\exp(tX)} \leq \exp\parenth{\frac{\subg^2t^2}{2}}
    \qtext{for all} 
t \in \reals.
\end{talign}
\end{definition}

For each $n\in\naturals$, we note that  $m_n$ %
is the associated thinning factor at the start of the $\express.\update(\x_n)$ call and define its truncation $q_n \defeq m_n \wedge \mbar$. 
Within the first $n$ \express update calls, $\halve_{m,m}$ is called at most twice and each $(\halve_{m,i})_{i=0}^{m\wedge \mbar-1}$ is called at most $3\cdot 4^{\parenth{m\wedge\mbar}-i-1}$ times for each even $m\in \{0,\dots,m_n\}$. 
Therefore, by our sub-Gaussian thinning assumptions and the union bound, there exists an increasing sequence of events $(\event[n])_{n\in\naturals}$ of respectively probability at least
\begin{talign}
1-\half\sum_{\textup{even } m=0}^{m_n}({2\delta_{m,m}+3\sum_{i=0}^{\parenth{m\wedge\mbar}-1}4^{\parenth{m\wedge\mbar}-i-1}\delta_{m,i}})
\end{talign}
on which, for each $n\in\naturals$, each  $\halve_{m,i}$ call invoked within the first $n$ update calls is $\subgh[n]\!(\nin)$ sub-Gaussian on $\event[n]$ given its input.

Now fix any $n\in\naturals$. %
If $n<4\nout$, then the \express approximation is exact and hence $\subge(n)=0$, so assume $n\geq4\nout$.
Our ultimate goal is to bound the $\event[n]$-sub-Gaussian constant,
\begin{align}
\subge(n)
    =
\sup_{f\in\rkhs: \knorm{f}=1} \textstyle\sqrt{\frac{2}{\knorm{f}^2}\log(\Esub{\event[n]}[\exp(\frac{1}{n}\sum_{i=1}^{n} f(\x_i) - \sum_{(\x,w)\in\xout(n)} w\,f(\x))])}. 
\end{align}
To achieve this, we fix $f\in\rkhs$ with $\knorm{f}=1$ and define, for each $(\kappa,\subg,\event)$-sub-Gaussian thinning algorithm \alg acting on an input dataset $\xin$, the \emph{unnormalized discrepancy}
\begin{talign}
\psi_{\alg}
    =
\nin(\frac{1}{\nin}\sum_{\x\in\xin} f(\x) - \sum_{(\x,w)\in\xout} w\,f(\x)).
\end{talign}
Notably, by \cref{def:alg-subg,def:subg}, $\psi_{\alg}$ is $\sigma=\nin\subg$ sub-Gaussian on $\event$ given $\xin$.

We will bound $\subge(n)$ by first controlling the $\event[n]$-sub-Gaussian constant $\sige(n)$ of the \express unnormalized discrepancy
\begin{talign}
\psie(n)
    \defeq
n(\frac{1}{n}\sum_{i=1}^{n} f(\x_i) -  \sum_{(\x,w)\in\xout(n)} w\,f(\x)).
\end{talign}
Our primary insight is that $\psie(n)$ decomposes into the error contributed by each \halve call and each stratified subsampling operation.  That is,
\begin{talign}
\psie(n)
    =
\sum_{\even\, m=0}^{m_n}
    &\big[
2^{m} \psi_{\halvetag_{m,m}}(4\nout)
    +
2^{m+1} \psi_{\halvetag_{m,m}}(2\nout) 
    +
\sum_{k=1}^{c_m}\psi_{\subsampletag_{m,k}} \indicator\brackets{m>\mbar}
\\
    &+
\sum_{i=0}^{(m\wedge \mbar)-1}
\sum_{k=1}^{c_{m,i}}
2^{i+m-(m\wedge \mbar)}\psi_{\halvetag_{m,i}}^{(k)}%
\big]
\end{talign}
where 
\begin{itemize}
    \item $\psi_{\halvetag_{m,m}}(4\nout)$ and $\psi_{\halvetag_{m,m}}(2\nout)$ are the unnormalized discrepancies associated with the first and second \express calls to $\halve_{m,m}$, 
    \item $c_{m,i}$ is the number of times $\halve_{m,i}$ is called within the first $n$ \express updates , 
    \item $\psi_{\halvetag_{m,i}}^{(k)}$ is the unnormalized discrepancy associated with the $k$-th $\halve_{m,i}$ call, 
    \item $c_m$ is the number of strata of size $N_m\defeq 2^{m-\mbar}$ previously processed or being processed by $\subsample(m,\mbar)$ by the end of the $n$-th update, 
    \item $\stratum\defeq(\x_{m,k,i})_{i=1}^{N_m}$ is the subsequence of points corresponding to the $k$-th stratum of size $N_m$ processed by $\subsample(m,\mbar)$,
    \item $n_{m,k}$ is the number of points from $\stratum$ processed by $\subsample(m,\mbar)$ by the end of the $n$-th update,
    \item $I_{m,k}\sim\Unif([N_m])$ is the index of the point selected by $\subsample(m,\mbar)$ from $\stratum$, and
    \item $\psi_{\subsampletag_{m,k}} = (\sum_{i=1}^{N_m}\indic{i\leq n_{m,k}}f(\x_{m,k,i}))-
    N_m\indic{I_{m,k}\leq n_{m,k}}f(\x_{m,k,I_{m,k}})$
    is the unnormalized discrepancy from subsampling $\stratum$ at the end of the $n$-th update.
\end{itemize}

Since $\indic{I_{m,k}\leq n_{m,k}}f(\x_{m,k,I_{m,k}})$ is a uniform subsample, \citet[(4.16)]{409cf137-dbb5-3eb1-8cfe-0743c3dc925f} and \citet[(4.8)]{Steinwart2008SupportVM} imply that $\psi_{\subsampletag_{m,k}}$ is sub-Gaussian with 
\begin{talign}
\sig_{\subsampletag_m}
    \leq 
N_m\max_{i\in[n]} |f(\x_i)|
    \leq 
N_m\knorm{f} \max_{i\in[n]} \sqrt{\kernel(\x_i,\x_i)}
    =
N_m\sqrt{\kin[n]}.
\end{talign}
Moreover, since each successive %
$\halve_{m,i}$ call is $\subgh\!(\nin)$ 
sub-Gaussian on $\event[n]$ given its input, each $\psi_{\halvetag_{m,m}}(\nin)$ is $\sig_{m}(\nin) = \nin\subgh\!(\nin)$ sub-Gaussian on $\event[n]$ given its input, and each $\psi_{\halvetag_{m,i}}^{(k)}$ is $\sig_{m,i} = 2^{i-(m\wedge \mbar)+2}\nout\subgh\!(2^{i-(m\wedge \mbar)+2}\nout)$ on $\event[n]$ given its input. 

Sub-Gaussian additivity \citep[Lem.~14]{dwivedi2024kernel} and the bounds $c_m \leq 3\cdot 2^{\mbar}\nout$ and $c_{m,i}\leq 3\cdot 4^{\parenth{m\wedge\mbar}-i-1}$ therefore imply that $\psie(n)$ is sub-Gaussian on $\event$ with 
\begin{talign}
&\sige^2(n) 
    \leq
\sum_{\even\, m=0}^{m_n}
    \big[
c_m\sig_{\subsampletag_m}^2\indic{m > \mbar}
    +
4^{m} \sig_{m}^2(4\nout)
    +
4^{m+1} \sig_{m}^2(2\nout) \\
    &\hspace{5\baselineskip}+
\sum_{i=0}^{(m\wedge \mbar)-1}
\sum_{k=1}^{c_{m,i}}
4^{i+m-(m\wedge \mbar)}\sig_{m,i}^2
\big] \\
    &\leq
\sum_{\even\, m=0}^{m_n}
    \big[
\frac{4^{m} 3 \kin[n]\nout\indic{m>\mbar}}{2^{\mbar}}
    + 
4^{m} \sig_{m}^2(4\nout)
    +
4^{m+1} \sig_{m}^2(2\nout) 
    +   
3\sum_{i=0}^{(m\wedge \mbar)-1}
4^{m-1}\sig_{m,i}^2
\big] \\
    &\leq
\frac{4^{m_n+2} 
\kin[n]\nout\indic{m_n>\mbar}}{2^{\mbar} 5}
    +
\sum_{\even\, m=0}^{m_n}
    4^{m+2} \nout^2\big[
\subgh[n]^2\!(4\nout)
    +
\subgh[n]^2\!(2\nout) \\
    &\hspace{10\baselineskip}+   
3\sum_{i=0}^{(m\wedge \mbar)-1}
4^{i-(m\wedge \mbar)-1}\subgh[n]^2\!(2^{i-(m\wedge \mbar)+2}\nout)
    \big]. \label{eq:express-sig}
\end{talign}
Since the final bound is independent of the arbitrary test function $f$, the same upper bound holds for $n^2\subge^2(n)$.

Now, define $\subg_n^2\defeq \subgh[n]^2\!(4\nout)+\subgh[n]^2\!(2\nout)$. 
Our monotonicity assumptions imply 
\begin{talign}
\subgh[n]^2\!(4\nout)
    +
\subgh[n]^2\!(2\nout)
    &\leq
\subg_n^2, 
    \qtext{and} \\
4^{i-(m\wedge \mbar)+2}\subgh[n]^2\!(2^{i-(m\wedge \mbar)+2}\nout)
    &\leq
4\subgh[n]^2\!(2\nout)
    \leq
4\subg_n^2
    \qtext{for all}
i \leq (m\wedge \mbar) - 1.
\end{talign}
Therefore, by the bound \cref{eq:express-sig}, the definition $K_n\defeq (m_n+2)/2$, 
and the constraint $\mbar \leq \log_2(2\nout)$,
\begin{talign}
\subge^2(n)
    &\leq 
\frac{4^{m_n+2} 
\kin[n]\nout\indic{m_n>\mbar}}{2^{\mbar} 5n^2}
    \!+\!
\frac{\indic{n\geq 4\nout}}{n^2}
\sum_{\even\, m=0}^{m_n}
    4^{m+2} \nout^2\subg_n^2(1 + \frac{3}{16}(m\wedge \mbar)) \\
    &\leq 
\frac{16^{K_n} 
\kin[n]\nout\indic{m_n>\mbar}}{2^{\mbar} 5n^2}
    \!+\!
\nout^2\subg_n^2(1+\frac{3}{16}q_n)
\frac{\indic{n\geq 4\nout}}{n^2}\sum_{k=1}^{K_n} 16^{k} \\
    &=
\frac{16^{K_n} 
\kin[n]\nout\indic{m_n>\mbar}}{2^{\mbar} 5n^2}
    \!+\!
\nout^2\subg_n^2(1 + \frac{3}{16}q_n)\frac{16}{15}\frac{16^{K_n}-1}{n^2}\\
    &\leq
\frac{16 
\kin[n]\indic{m_n>\mbar}}{5\nout 2^{\mbar}}
    \!+\!
\subg_n^2(16 + 3q_n)\frac{16}{15}
    \leq
\frac{16 
\kin[n]\indic{m_n>\mbar}}{5\nout 2^{\mbar}}
    \!+\!
\subg_n^2(16 + 3\log_2(2\nout))\frac{16}{15} \\
    &=
\frac{16 
\kin[n]\indic{m_n>\mbar}}{5\nout 2^{\mbar}}
    \!+\!
\subg_n^2(1 + 3\log_2(64\nout))\frac{16}{15}
    \leq
\frac{16 
\kin[n]\indic{m_n>\mbar}}{5\nout 2^{\mbar}}
    \!+\!
4\log_2(64\nout)\subg_n^2.
\end{talign}

\section{\pcref{khexpress-guarantees}}
\label{proof:khexpress-guarantees}

\paragraph{Proof of runtime claim~\cref{eq:khexpress-runtime}.}
The \khd implementation used in \khexpressd has runtime $r_{\halvetag}(\nin)=O(d\nin^2)$.
Hence, the runtime claim follows immediately from \cref{rmk:express-polynomial-runtime} with $\tau=2$.

\paragraph{Proof of space claim.}
\cref{express-guarantees} and the $s_{\halvetag}(\nin)=\nin d$ space complexity of \khd imply that
\begin{talign}
s_{\expresstag}(n)\leq 6\nout d+s_{\halvetag}(4\nout) \leq 10\nout d.
\end{talign}

\paragraph{Proof of quality claim~\cref{eq:khexpress-error}.}
For each $n < 4\nout$, $\xout(n)$ represents the first $n$ points exactly, so $\subge(n)=0$.
For each $n\geq 4\nout$, let $m_n$ be the associated thinning factor at the start of the $\express.\update(\x_n)$ call, and let $\dstar[n]$ represent the smallest failure probability used by any \kh call within the first $n$ updates.  
Since $\delta_m$ is decreasing in $m$ and $\mbar\geq 1$, we have 
\begin{talign}
\dstar[n]
    &\geq
\min(\half,\frac{4^{1-(m_n\wedge\mbar)}}{3(m_n\wedge\mbar)})\delta_{m_n}
    \geq
\frac{4}{3\mbar 4^{\mbar}}\delta_{m_n}
    =
\frac{\delta}{2}\frac{4}{3\mbar 4^{\mbar}}
\frac{\log_2(1+2/(m_n+4))}{\log_2(m_n/2+2)\log_2(m_n/2+3)}\\
    &\geq
\frac{4\delta}{3\mbar 4^{\mbar}(m_n+5)\ln(2)\log_2(m_n/2+2)\log_2(m_n/2+3)},
\label{dstar-lower-bound}
\end{talign}
where the final inequality follows from \cref{lem:log2-one-plus-lower}.
\begin{lemma}[Logarithm lower bound]\label{lem:log2-one-plus-lower}
For all $x>0$,
$\log_2(1+x)
    \geq
\frac{2}{(1+2/x)\ln(2)}.$
\end{lemma}
\begin{proof}
Let $h(x)\defeq\ln(1+x)-\frac{2x}{2+x}$ for $x>-1$. Since $h(0)=0$ and
\begin{talign}
h'(x)
    =
\frac{1}{1+x}-\frac{4}{(2+x)^2}
    =
\frac{x^2}{(1+x)(2+x)^2}
    \geq 0
\qtext{for all}
    x > 0
\end{talign}
we have $h(x)\geq 0$ for all $x>0$. Hence $\log_2(1+x)\ln(2) = \ln(1+x)\geq\frac{2x}{2+x}$.
\end{proof}

Now, for each $\eta\in(0,1)$ and $n\geq 4\nout$, 
Prop.~B.2 of \citet{carrell2025low} implies that $\kh(\eta,\kernel_{\xin})$ is $(\kernel_{\xin},\subgh[n]\!(\nin,\eta),\eta)$-sub-Gaussian with 
\begin{talign}
\subgh[n]^2\!(\nin,\eta)
    \defeq
\frac{4a_n\log(2\nin/\eta)}{\nin^2}
\end{talign}
on any input $\xin$ of size $\nin$ from $\xset_n=(\x_i)_{i=1}^n$. 
Since each $\kernel_{\xin}\psdge\kernel_{\xset_n}$, the following lemma ensures that $\kh(\eta,\kernel_{\xin})$ is $(\kernel_{\xset_n},\subgh[n]\!(\nin,\eta),\eta)$-sub-Gaussian as well.

\begin{lemma}[Sub-Gaussian inheritance]
If $\kernel'\psdge\kernel$ and \alg is $(\kernel',\subg)$-sub-Gaussian on $\event$ given $\xin$, then \alg is also $(\kernel,\subg)$-sub-Gaussian on $\event$ given $\xin$.
\end{lemma}
\begin{proof}
Fix any $f\in\rkhs[\kernel']$. 
By the assumed sub-Gaussianity, 
\begin{talign}
\Esub{\event}\brackets{\exp\parenth{ \frac{1}{\nin}\sum_{i=1}^{\nin} f(\x_i) - \sum_{(\x,w)\in\xout} w\,f(\x)}} 
    \leq 
\exp\parenth{\frac{\subg^2}{2} \norm{f}_{\kernel'}^2}.
\end{talign}
Since 
$\kernel'\psdge\kernel$, we further have $f\in\rkhs$ with $\knorm{f} \geq \norm{f}_{\kernel'}$ \citep[Thm.~1]{aronszajn1950reproducing}, yielding the result.
\end{proof}

Setting $\subgh[n]^2\!(\nin) \defeq \subgh[n]^2\!(\nin,\dstar[n])$, we find that $\nin\subgh[n]\!(\nin)$ is nondecreasing in $\nin$ and $n$. Therefore, by \cref{express-quality}, the \express $(\kernel_{\xset_n},\subg_{\expresstag}(n))$-sub-Gaussianity bound 
\begin{talign}
\subg_{\expresstag}^2(n)
    &\defeq
\frac{16}{15}\parenth{16+3(m_n\wedge\mbar)}
\,(\subgh[n]^2\!(4\nout) + \subgh[n]^2\!(2\nout))
    +
\frac{16\norm{\kernel_{\xset_n}}_n\indicator\brackets{m_n > \mbar}}{5\nout2^{\mbar}} \\
    &\leq
\frac{4}{3}\parenth{16+3(m_n\wedge\mbar)}
\,\frac{a_n\log(8\nout/\dstar[n])}{\nout^2}
    +
\frac{16a_n\indicator\brackets{m_n > \mbar}}{5\nout2^{\mbar}}
\end{talign}
holds for all $n\in\naturals$
on a common event $\event = \cup_{n\geq 1}\event[n]$ of probability at least
\begin{talign}
1-\sum_{\textup{even } m=0}^{\infty}({\frac{\delta_m}{2}+\frac{3}{8}\sum_{i=0}^{\parenth{m\wedge\mbar}-1}4^{\parenth{m\wedge\mbar}-i}\delta_{m,i}})
    &=
1-\sum_{\textup{even } m=0}^{\infty} \delta_m \\
    = 
1-\frac{\delta}{2}\sum_{\textup{even } m=0}^{\infty}
\big(\frac{1}{\log_2(m/2+2)}-\frac{1}{\log_2(m/2+3)}\big)
    &=
1-\frac{\delta}{2}.
\end{talign}
The claim~\cref{eq:khexpress-error} now follows from the estimates \cref{dstar-lower-bound}, $\mbar\leq \log_2(\nout)$, and $6\ln(2)\leq 5$.
\section{\pcref{att-err}}
\label{proof:att-err}

Consider running Thinformer Express (\cref{algo:thinformer}) on the sequence $((\query_i,\key_i,\val_i))_{i\geq1}$, and note that $\katt[\xset]\psdge\katt[\xset']$ whenever $\xset\subseteq\xset'$. 
Hence, 
\cref{khexpress-guarantees} implies that, for all $n\in\naturals$, the weighted coreset $\xout(n)$ obtained by updating \khexpressd with $\xset_n\defeq((\key_i,\val_i))_{i=1}^n$ is $(\katt[\xset_n],\subg_{\expresstag}(n))$-sub-Gaussian on an event $\event$ of probability at least $1-\delta/2$ where $\subge(n)$ satisfies %
\begin{talign}\label{eq:katt-subg}
\!\!\!\!\!\!\!\!\!\!\subg_{\expresstag}^2(n)
    &\leq
\frac{4\parenth{16+3(m_n\wedge\mbar)}\indic{n\geq 4\nout}}{3}
\frac{a_n\log(5\nout\mbar(m_n+5)\log_2^2(m_n/2+3)4^{\mbar}/\delta)}{\nout^2}
    +
\frac{16a_n\indicator\brackets{m_n > \mbar}}{5\nout2^{\mbar}} \\
    &=
O\parenth{\frac{a_n\log(\nout)\log(\nout\log(\frac{n}{\nout})/\delta)}{\nout^2}
    +
\frac{a_n}{2^{\mbar}\nout}}
\label{eq:katt-subg-O}
\end{talign}
with  
\begin{align}
a_n 
    &= 
\max_{\xin\subseteq\xset_n, \x\in\xin}  \katt[\xin](\x,\x)
    =
\max_{\xin\subseteq\xset_n, (\key,\val)\in\xin}  e^{\twonorm{\key}^2/\sqrt{d}}(1+\twonorm{\val}^2/\max_{(\key',\val')\in\xin} \infnorm{\val'}^2)\\
    &\leq
\max_{\xin\subseteq\xset_n, (\key,\val)\in\xin}  e^{\twonorm{\key}^2/\sqrt{d}}(1+\twonorm{\val}^2/ \infnorm{\val}^2)
    \leq
e^{r_n^2/\sqrt{d}}(1+d)
\label{eq:an_bound}
\end{align}
by \Holder's inequality and the definition of $r_n$.  

Since \ours attends over $(\query_n,\key_n,\val_n)$ prior to updating its \khexpress cache $C$ with $(\key_n,\val_n)$, \wattn assigns the new token weight $2^{-C.m}$ in a separate singleton level $\cset_{-1}$. Multiplying all attention weights by $2^{C.m}$ leaves the attention ratio unchanged and gives the new token unit weight in the unnormalized coreset. Its contribution therefore cancels exactly from the unnormalized discrepancy, leaving only the discrepancy of the preceding $n-1$ tokens. Consequently, the weighted coreset employed by the $n$-th \wattn call is $(\katt[\xset_n], \subgw(n))$-sub-Gaussian on $\event$ with $\subgw(n)=\frac{n-1}{n}\subge(n-1)$.

Now fix any $n\in\naturals$. 
The sub-Gaussian attention analysis of 
\citet[Thm.~1 and Lems.~F.1, F.2, and F.3]{carrell2025low} 
implies that, on $\event$, 
\begin{talign}\label{eq:subg-attout}
\infnorm{\attouthat_n - \attout_n}
    &\leq
2\subgw(n) \frac{\exp(\twonorm{q_n}^2/(2\sqrt{d}))}{\frac{1}{n}\sum_{i\in[n]}\exp(\inner{q_n}{k_i}/\sqrt{d})} \sqrt{2\log(2(d+1)/\delta_n)}\cdot\maxnorm{\V_n}
\end{talign}
with probability at least $1-\delta_n$.
The advertised inequality now follows from the estimates \cref{eq:katt-subg,,eq:subg-attout,,eq:an_bound}; the relations $\subgw(n) = \frac{n-1}{n}\subge(n-1)$, $m_{n-1}\leq m_n$, and $\kattin[n-1]\leq \kattin$; and the union bound as $\sum_{n=1}^\infty \delta_n = \delta/2$. Thus the bound holds simultaneously for all $n>4\nout$ with probability at least $1-\delta/2-\sum_{n=1}^\infty\delta_n=1-\delta$. For $n\leq4\nout$, the cache used by \wattn contains all preceding key-value pairs, so the attention output is exact on the same event.
The final big O expression additionally uses the estimate \cref{eq:katt-subg-O} and the inequality
\begin{talign}
\delta_n
    &=
\frac{\delta}{2}\frac{\log_2(1+1/(n+1))}{\log_2(n+1)\log_2(n+2)}
    \geq
\frac{\delta}{(2n+3)\ln(2)\log_2(n+1)\log_2(n+2)}
\end{talign}
justified by \cref{lem:log2-one-plus-lower}.

\section{\pcref{khexpress_query_cost}}
\label{proof:khexpress_query_cost}

\begin{proposition}[Cost of \wattn]
\label{khexpress_query_cost}
When attending over $(\query_i,\key_i,\val_i)_{i=1}^n$ with \ours (\cref{algo:thinformer}), the runtime contributed by \wattn satisfies, for $m_n \defeq 2\ceil{\log_4(n/(4\nout))}$,
\begin{talign}\label{eq:streaming_express_query}
\!\!\!r_{\wattntag}(n)
    \leq 
\begin{cases}
\frac{d}{2} n (n+1) & \text{if } n \leq 4\nout, \\
\begin{aligned}
&\frac{d}{2} (4\nout)(4\nout+1)
    +
\parenth{\frac{7}{2}\nout^2-\half\nout}(2^{m_n+2}-4)
\\
&\quad
    -
3\frac{m_n}{2}\nout^2
\end{aligned}
    & \text{if } 4\nout < n \text{ and } m_n\leq\mbar, \\
\frac{d}{2} (4\nout)(4\nout+1) + 6\nout(n-4\nout)
    & \text{if } 4\nout < n \text{ and } m_n>\mbar.
\end{cases}
\end{talign}
\end{proposition}

\begin{proof}
When $n\leq 4\nout$, the cache passed to the $i$-th \wattn call contains the preceding $i-1$ key-value pairs exactly, and \cref{algo:wattn} incorporates the input token before computing attention.
Thus the $i$-th query attends over $i$ points exactly, and the total query cost is $r_{\wattntag}(n)=\sum_{i=1}^n i=\half n(n+1)$.
Now suppose $n>4\nout$, so $m_n+2=2\ceil{\log_4(n/\nout)}$.
We first consider the case $m_n\leq\mbar$.
Since the total query cost is non-decreasing in $n$, it suffices to consider endpoints $n=\nout 2^s$ for $s\in 2\naturals$ with $2\leq s\leq\mbar+2$.
Let $R_s\defeq r_{\wattntag}(\nout 2^s)$ for even $s$.
For $2\leq s\leq\mbar$, the round from $\nout 2^s$ to $\nout 2^{s+2}$ consists of three sibling blocks, each containing $\nout 2^s$ queries and no subsampling.
By \cref{express_cache_size}, the query cost during sibling $j\in\braces{1,2,3}$ after $u$ points have been added to the current \compresstwo instance is
\begin{talign}
    j\nout+a_u+\sum_{\ell=1}^{s-1}\nout 2^{\ell-s}d_{u,\ell},
\end{talign}
where $u=\nout 2^{2-s}x_u+a_u$, $0\leq a_u<\nout 2^{2-s}$, and $x_u=\sum_{\ell=1}^{s-1}4^{\ell-1}d_{u,\ell}$.
Therefore,
\begin{talign}
R_{s+2}-R_s
    &=
\sum_{j=1}^3\sum_{u=0}^{\nout 2^s-1}
\parenth{j\nout+a_u+\sum_{\ell=1}^{s-1}\nout 2^{\ell-s}d_{u,\ell}}.
\end{talign}
The completed \express and sibling summaries contribute
\begin{talign}
\sum_{j=1}^3\sum_{u=0}^{\nout 2^s-1}j\nout
    =
6\nout^2 2^s.
\end{talign}
For the current \compresstwo cache, each value of $x_u\in\braces{0,\ldots,4^{s-1}-1}$ appears with $\nout2^{2-s}$ values of $a_u$, so
\begin{talign}
\sum_{u=0}^{\nout2^s-1}a_u
    &=
2\nout^2-\half\nout2^s,
\end{talign}
and, for each digit place $\ell$, the digits $0,1,2,3$ appear equally often, giving
\begin{talign}
\sum_{u=0}^{\nout2^s-1}d_{u,\ell}
    =
\frac{3}{2}\nout2^s.
\end{talign}
Combining these identities yields
\begin{talign}
\sum_{u=0}^{\nout2^s-1}
\parenth{a_u+\sum_{\ell=1}^{s-1}\nout 2^{\ell-s}d_{u,\ell}}
    =
\parenth{\frac{3}{2}2^s-1}\nout^2-\half\nout2^s.
\end{talign}
Hence
\begin{talign}
R_{s+2}-R_s
    &=
\frac{21}{2}\nout^2 2^s-3\nout^2-\frac{3}{2}\nout2^s.
\end{talign}
Unrolling over $s=2,4,\ldots,m_n$ gives
\begin{talign}
R_{m_n+2}
    &=
\half(4\nout)(4\nout+1)
    +
\sum_{k=1}^{m_n/2}\braces{\frac{21}{2}\nout^2 2^{2k}-3\nout^2-\frac{3}{2}\nout2^{2k}} \\
    &=
\half(4\nout)(4\nout+1)
    +
\parenth{\frac{7}{2}\nout^2-\half\nout}(2^{m_n+2}-4)
    -
3\frac{m_n}{2}\nout^2.
\end{talign}
Since $r_{\wattntag}(n)\leq R_{m_n+2}$, this proves the second case of \cref{eq:streaming_express_query}.
For the remaining case, $m_n>\mbar$ and each query after the exact phase attends over at most the weighted cache returned by \express, which has size at most $6\nout$ by \cref{max_cache_size}.
Thus
\begin{talign}
r_{\wattntag}(n)
    \leq
\half(4\nout)(4\nout+1)+6\nout(n-4\nout),
\end{talign}
which proves the final case.
\end{proof}

\section{\oexpress}
\label{app:express_recursive_storage_cost}

In this section we introduce and analyze an offline version of \express that can be used when the input points $(\x_i)_{i=1}^n$ are all observable in advance.
\subsection{\osubsample}
\label{sec:subsample-recursive}

The streaming \subsample algorithm of \cref{algo:subsample} is equivalent to the following offline procedure (detailed in \cref{algo:subsample-recursive}): partition the input sequence of length $2^m\nout$ into $2^{m \wedge \mbar}\nout$ equal-sized groups and pick one uniformly random point from each group.

\begin{algorithm2e}[H]
\caption{\small \osubsample%
}
\label{algo:subsample-recursive}
\small
\KwIn{Halving level $m \in \naturals$, inflation factor $\mbar\in\naturals$, $\nin$ input points $\xin$}%
\KwOut{Stratified subset of size $\min(\nin, \nin/2^{m-\mbar})$}
\BlankLine
\lIf{$m \le \mbar$}{\KwRet{$\xin$}}
Partition $\xin$ into $G$ groups $(\coreset[m,i]_{\subsampletag})_{i=1}^G$, each of size $2^{m - \mbar}$\\
\For{$i = 1, \ldots, G$}{
    $s_i \gets$ uniformly draw one point from $\coreset[m,i]_{\subsampletag}$
}
\KwRet{$(s_i)_{i=1}^G$}
\end{algorithm2e}

\subsection{\ocompresstwo}
\label{sec:compress2-recursive}
The streaming \compresstwo algorithm of \cref{algo:compress2} is equivalent to the following recursive offline procedure (detailed in \cref{algo:compress2-recursive}): partition the input into $4^{m-1}$ groups of equal size, halve each group, concatenate the outputs, and recurse. The base case ($m=0$) returns the input unchanged.
\begin{algorithm2e}[H]
\caption{\small \ocompresstwo}%
\label{algo:compress2-recursive}
\small
\KwIn{Halving level $m \in \naturals$, $\nin=2^m\nout$ input points $\xin$, $\halve=\halve_{0:(m-1)}$}
\KwOut{Coreset of size $\nout$}
\BlankLine
\lIf{$m = 0$}{\KwRet{$\xin$}}
Partition $\xin$ into $4^{m-1}$ groups $(\coreset[m,k])_{k=1}^{4^{m-1}}$, each of size $\nin / 4^{m-1}$\\
\KwRet{$\ocompresstwo\big(m-1,\, \bigcup_{k=1}^{4^{m-1}} \halve_0(\coreset[m,k]),\, \halve_{1:(m-1)}\big)$}
\end{algorithm2e}

\subsection{Offline-\express}\label{sec:express-recursive}

The streaming \express algorithm of \cref{algo:express} is equivalent to the recursive offline procedure detailed in \cref{algo:express-recursive}.

\begin{algorithm2e}[H]
\caption{\small \oexpress}
\label{algo:express-recursive}
\small
\KwIn{Target cache size $\nout$, inflaction factor $\mbar$, $\nin=2^{m+2}\nout$ input points $\xin$ with even $m\geq -2$, $\halve$}%
\KwOut{Coreset of size $\nout$}
\BlankLine

\lIf{$\nin \leq \nout$}{\KwRet{$\xin$} \hfill// {Exact phase: preserve the first $\nout$ points}
}

\BlankLine
Partition $\xin$ into $4$ groups $(\inputcoreset^{(m,j)})_{j=0}^{3}$, each of size $\nin / 4$\\

\BlankLine
// {Process first group with \oexpress}\\
$\coreset[m]_{\expresstag} \gets \oexpress(\nout, \mbar, \inputcoreset^{(m,0)}, \halve)$ 

\BlankLine

// {Thin phase: process remaining groups with \osubsample + \ocompresstwo} \\
\SetKwFor{ParallelFor}{parallel for}{do}{end}
\ParallelFor{$j = 1,2,3$}{
    $\coreset[m,j] \gets \hyperref[algo:subsample-recursive]{\color{black}{\osubsample}}(m, \mbar, \nout, \inputcoreset^{(m,j)})$ \tcp*{$|\coreset[m,j]| = 2^{m \wedge \mbar}\nout $} %
    $\coreset[m,j]_{\compresstag} \gets \ocompresstwo(m \wedge \mbar,\, \coreset[m,j],\, \halve_m)$ \tcp*{$|\coreset[m,j]_{\compresstag}| = \nout$} %
    \label{line:express-recursive-compress2}
}

\BlankLine
// {\halve phase}\\
\KwRet{$\coreset[m+2]_{\expresstag}\defeq\halve_{m,m}(\halve_{m,m}(\coreset[m]_{\expresstag} \cup \coreset[m,1]_{\compresstag} \cup \coreset[m,2]_{\compresstag} \cup \coreset[m,3]_{\compresstag}))$ \tcp*{$|\coreset[m+2]_{\expresstag}| = \nout$}}\label{line:express-recursive-warm-halve}
\end{algorithm2e}

When the input points $\xin$ are observed in advance, we can implement \cref{algo:express-recursive} by processing the three \ocompresstwo calls in each round in parallel, while also running the \halve calls at each level of \cref{algo:compress2-recursive} in parallel.
The next result bounds the storage cost of such a procedure.
\begin{proposition}[Storage cost of \oexpress]
\label{express_recursive_storage_cost}
Suppose \halve with input size $\nin$ uses at most $s_{\halvetag}(\nin)$ space for nondecreasing $s_{\halvetag}$.
Excluding input point storage, \oexpress (\cref{algo:express-recursive}) can be implemented using at most
\begin{talign}\label{eq:offline_express_storage_cost}
s_{\expresstag}(\nin) \leq
\indic{\nin > \nout} (4\nout + h_{\nin})
\end{talign}
space where $m_{n} \defeq 2\ceil{\log_4(n/(4\nout))}$, $q_n \defeq m_n\wedge\mbar$, and
\begin{talign}
h_{n}
    \defeq
\max\braces{
    \max_{0\leq q\leq q_{n}}\max_{0\leq i<q} 3\cdot 4^{q-1-i}s_{\halvetag}(\nout 2^{2-q+i}),
    s_{\halvetag}(4\nout)
}.
\end{talign}
Here an empty maximum is taken to be zero.
\end{proposition}
\begin{proof}
In the exact base case (with $\nin\leq \nout$), \cref{algo:express-recursive} uses no additional storage.
Now suppose $\nin>\nout$.
In the first round, $\nin=4\nout$ and $m=q_{\nin}=0$, so each \ocompresstwo call returns its input unchanged and uses no halving workspace.
Thus all input points are preserved until the final \halve phase of this round, which still performs two halving calls, reducing $4\nout$ points to $\nout$.
Consequently, $h_{4\nout}=s_{\halvetag}(4\nout)$, and the zero-storage base case for the complete \oexpress call remains $\nin\leq\nout$.
The recursive \oexpress call on the first group finishes before the three sibling calls begin, so their workspaces do not overlap with that of the recursive call.
Prior to the \halve phase, \cref{algo:express-recursive} stores the \oexpress coreset $\coreset[m]_{\expresstag}$ and the three completed \ocompresstwo sibling coresets $\coreset[m,j]_{\compresstag}$ using pointers to at most $4\nout$  data points.
Run the three sibling calls level by level, reusing their storage between successive levels.
For a \ocompresstwo call at level $q\leq q_{\nin}$, recursion level $i$ runs $4^{q-1-i}$ calls to \halve in parallel, each with input size $\nout 2^{2-q+i}$.
Across the three sibling calls, there are therefore $3\cdot 4^{q-1-i}$ simultaneous \halve calls, so the parallel compression workspace is bounded by
\begin{talign}
\max_{0\leq q\leq q_{\nin}}\max_{0\leq i<q} 3\cdot 4^{q-1-i}s_{\halvetag}(\nout 2^{2-q+i}).
\end{talign}
The two \halve calls made during the \halve phase run sequentially after the sibling calls finish and have input sizes $4\nout$ and $2\nout$, so the monotonicity of $s_{\halvetag}$ bounds their temporary workspace by $s_{\halvetag}(4\nout)$.
Taking the maximum over these phases and all rounds, then adding the stored summaries, gives \cref{eq:offline_express_storage_cost}.
\end{proof}
\begin{corollary}[Offline storage cost with linear-space \halve]
If \kh uses $s_{\halvetag}(\nin)=\nin$ space (e.g., to store pointers to its input points), then, apart from input point storage, \cref{algo:express-recursive} can be implemented using at most
\begin{talign}
4\nout + \max(3\cdot 2^{m_{\nin}\wedge \mbar} ,4)\nout
\end{talign}
space where $m_{n} \defeq 2\ceil{\log_4(n/(4\nout))}$.
\end{corollary}
\begin{proof}
Substituting $s_{\halvetag}(\nin)=\nin$ into the parallel compression bound gives
\begin{talign}
3\cdot 4^{q-1-i}\nout 2^{2-q+i}
=3\cdot 2^{q-i}\nout
\leq 3\cdot 2^{m_{\nin}\wedge\mbar}\nout.
\end{talign}
The final \halve phase uses at most $4\nout$ workspace, yielding the stated bound.
\end{proof}

\section{Supplementary Experiment Details}
\label{sec:experiment_supplement}

Each experiment was run on an Ubuntu 24.04.4 LTS node with $4$ Intel Xeon Gold 5218 cores @ 2.30 GHz, 64 GB RAM, and a single NVIDIA RTX A6000 GPU (48 GB memory, CUDA 12.8, driver version 570.124.06). 

\subsection{Accelerating prefill}
\label{app:attention_configs}

The accelerating prefill experiments of \cref{fig:context-runtime} were run using Python 3.12.11 and PyTorch 2.8.0.
The implementation and default settings of HyperAttention were taken from \texttt{hyper-attn},\footnote{\url{https://github.com/insuhan/hyper-attn}} and our experiment builds on this open-source repository.
For the LongBench-E prefill experiment in \cref{fig:masked-prefill-32k}, we use \texttt{zai-org/chatglm2-6b-32k}.\footnote{\url{https://huggingface.co/zai-org/chatglm2-6b-32k}}
For the experiment in \cref{fig:maxnorm}, we sample the queries, keys, and values from the embedding matrix of the ChatGLM2-6B-32K model and report the median of the scaled maximum error $\infnorm{\attouthat_n - \attout_n} / \sqrt{\log n \log\log n}$ across the 32 heads.
For the single-layer runtime experiment in \cref{fig:unmasked-prefill}, we replace the second attention layer in ChatGLM2-6B-32K with Thinformer and measure prefill runtime over varying context lengths.
We report speedup of the methods over \texttt{torch.nn.functional.scaled\_dot\_product\_attention} with the FlashAttention 2 backend.
We use $100$ warm-up runs and $20$ timed runs for all accelerating prefill experiments.
An ablation with $\delta=0.01$ and \texttt{n\_sinks=0} can be found in \cref{fig:masked-prefill-32k-small-delta,fig:masked-prefill-32k-no-sink-window}, respectively.

\begin{figure}[h!]
\centering
\includegraphics{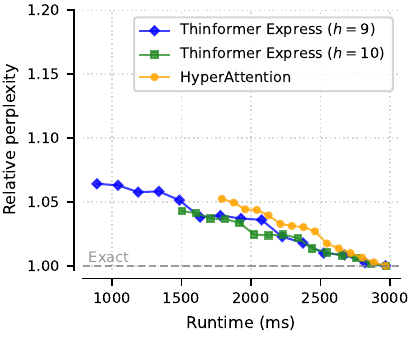}
\caption{\tbf{Accelerating long-context prefill $\delta$ ablation.} We repeat the $n=32768$ masked attention experiment of \cref{fig:masked-prefill-32k} using \ours with $\delta=0.01$ instead of the $\delta=\half$ default and observe nearly identical runtime-vs.-perplexity improvements over HyperAttention.
}
\label{fig:masked-prefill-32k-small-delta}
\end{figure}

\begin{figure}[t!]
\centering
\includegraphics[width=0.49\textwidth]{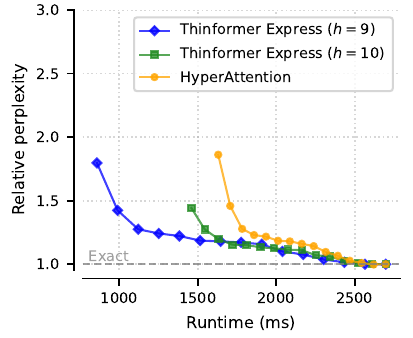}
\caption{\tbf{Accelerating long-context prefill sink ablation.} We repeat the $n=32768$ masked attention experiment of \cref{fig:masked-prefill-32k} with \texttt{n\_sinks=0} instead of \texttt{n\_sinks=32} for all methods and find that both settings of \ours ($\errsymb=9$ and $\errsymb=10$) still offer a better runtime-vs.-perplexity trade-off than HyperAttention.}
\label{fig:masked-prefill-32k-no-sink-window}
\end{figure}

\begin{table}[h!]
    \centering
    \caption{\tbf{Configurations for the attention methods of \cref{fig:masked-prefill}.}}
    \begin{tabular}{cc}
        \toprule
        {\bf Attention Algorithm} & {\bf Configuration}
        \\\midrule
        \multirow{1}{*}{\bf HyperAttention} &
        \begin{tabular}[c]{@{}c@{}}
        \texttt{lsh\_num\_projs=7},
        \texttt{block\_size=256}, \\
        \texttt{sample\_size=256},
        \texttt{min\_seq\_len=4096}, \\
        \texttt{n\_sinks=0},
        \texttt{n\_window=0}
        \end{tabular}
        \\[1mm]
        \multirow{1}{*}{\bf \ours (Ours)} &
        \begin{tabular}[c]{@{}c@{}}
        \texttt{n\_sinks=0},
        \texttt{n\_window=0}
        \end{tabular}
        \\
        \bottomrule
    \end{tabular}
    \label{tab:masked-prefill-configs}
\end{table}

\begin{table}[h!]
    \centering
    \caption{\tbf{Configurations for the attention approximation methods of \cref{fig:masked-prefill-32k}.}}
    \begin{tabular}{cc}
        \toprule
        {\bf Attention Algorithm} & {\bf Configuration}
        \\\midrule
        \multirow{1}{*}{\bf HyperAttention} &
        \begin{tabular}[c]{@{}c@{}}
        \texttt{lsh\_num\_projs=7},
        \texttt{block\_size=256}, \\
        \texttt{sample\_size=256},
        \texttt{min\_seq\_len=4096}, \\
        \texttt{n\_sinks=32},
        \texttt{n\_window=32}
        \end{tabular}
        \\[1mm]
        \multirow{1}{*}{\bf \ours (Ours)} &
        \begin{tabular}[c]{@{}c@{}}
        \texttt{n\_sinks=32},
        \texttt{n\_window=32}
        \end{tabular}
        \\
        \bottomrule
    \end{tabular}
    \label{tab:attention-configs}
\end{table}

\begin{table}[h!]
    \centering
    \caption{\tbf{Configurations for the attention methods of \cref{fig:unmasked-prefill}.}}
    \begin{tabular}{cc}
        \toprule
        {\bf Attention Algorithm} & {\bf Configuration}
        \\\midrule
        \multirow{1}{*}{\bf Thinformer (Original)} &
        \begin{tabular}[c]{@{}c@{}}
        $\nout=256$
        \end{tabular}
        \\[1mm]
        \multirow{1}{*}{\bf Thinformer (Triton)} &
        \begin{tabular}[c]{@{}c@{}}
        $\nout=256$
        \end{tabular}
        \\
        \bottomrule
    \end{tabular}
    \label{tab:unmasked-prefill-configs}
\end{table}

\subsection{Accelerating KV cache compression}
\label{app:kv_cache_configs}

The KV cache compression experiments of \cref{fig:kv-cache-compression} were run using Python 3.13.9 and PyTorch 2.10.0.
We use \texttt{meta-llama/Llama-3.1-8B-Instruct} for all KV cache compression experiments.\footnote{\url{https://huggingface.co/meta-llama/Llama-3.1-8B-Instruct}}
We use greedy decoding for all methods and set \texttt{max\_new\_tokens} per dataset following the LongBench-E settings of \citet{devoto2025expectedattention} with one exception.
Namely, we apply the chat template for HotpotQA-E and TriviaQA-E and disable it for TREC-E: with Llama-3.1-8B-Instruct, we found that disabling the chat template on TREC-E outperformed enabling it.
The implementations and default settings of all baseline methods were taken from \texttt{kvpress} version 0.3.0,\footnote{\url{https://github.com/NVIDIA/kvpress/tree/main}} and our experiment builds on this open-source repository.

During prefill, PyramidKV, StreamingLLM, and SnapKV use \texttt{torch.nn.functional.scaled\_dot\_product\_attention} with the FlashAttention 2 backend to compute the keys and values exactly; PyramidKV-Express, StreamingLLM-Express, and SnapKV-Express use \ours with \texttt{n\_sink=32} and \texttt{n\_window=32} to compute the approximate context keys and values.
Runtime in \cref{fig:kv-cache-compression} is the sum of key-value computation during prefill, KV cache compression, and query attention during decoding over all $32$ layers.
We use one warm-up run and one timed run for all accelerating KV cache compression experiments.
\begin{table}[h!]
    \centering
    \caption{\tbf{Configurations for the attention approximation methods of \cref{fig:kv-cache-compression}.}}
    \begin{tabular}{cc}
        \toprule
        {\bf Attention Algorithm} & {\bf Configuration}
        \\\midrule
        \multirow{1}{*}{\bf PyramidKV} &
        \begin{tabular}[c]{@{}c@{}}
        \texttt{window\_size=64},
        \texttt{kernel\_size=5}, \\
        \texttt{beta=20}
        \end{tabular}
        \\[1mm]
        \multirow{1}{*}{\bf StreamingLLM} & \texttt{n\_sink=4}
        \\[1mm]
        \multirow{1}{*}{\bf SnapKV} &
        \begin{tabular}[c]{@{}c@{}}
        \texttt{window\_size=64},
        \texttt{kernel\_size=5}
        \end{tabular}
        \\
        \bottomrule
    \end{tabular}
    \label{tab:kv-cache-configs}
\end{table}

\subsection{Long-form decoding}
\label{app:long_form_configs}

The long-form decoding experiments of \cref{fig:long-form} were run using Python 3.13.9 and PyTorch 2.10.0.
For the cache size-accuracy trade-off curve, we sweep over the target cache sizes $\{250, 500, 1\mathrm{k}, 1.5\mathrm{k}, 2\mathrm{k}, 3\mathrm{k}, 4\mathrm{k}\}$ for all non-\ours baselines. 
To form the curves in \cref{fig:long-form} (left), we vary the threshold and compute the average accuracy and realized cache size over all math problems, where for each threshold and problem, we use the target cache size (for alternative methods) or the $\errsymb$ value (for \ours) with the largest realized cache size less than the given threshold.
For \ours, we sweep over $\errsymb \in \{6, 7, 8, 10, 12\}$, corresponding to target cache sizes $\nout=2^{\errsymb}$.
For the accuracy-time trade-off curve, we sweep over the target cache sizes $\{1\mathrm{k}, 1.5\mathrm{k}, 2\mathrm{k}, 3\mathrm{k}, 4\mathrm{k}\}$ for Expected Attention, SnapKV, and KeyDiff and $\{1.5\mathrm{k}, 2\mathrm{k}, 3\mathrm{k}, 4\mathrm{k}\}$ for StreamingLLM and Knorm.
We generate with \texttt{max\_new\_tokens=32768} and use the recommended sampling parameters of \texttt{temperature=0.6} and \texttt{top\_p=0.95} from the official DeepSeek-R1-Distill-Llama-8B model card.\footnote{\url{https://huggingface.co/deepseek-ai/DeepSeek-R1-Distill-Llama-8B}}
We evaluate on the MATH-500 dataset\footnote{\url{https://huggingface.co/datasets/HuggingFaceH4/MATH-500}} \citep{hendrycks2021measuring,lightman2023let} and score answers using the grader from the PRM800K codebase\footnote{\url{https://github.com/openai/prm800k}} of \citet{lightman2023let}.

We use \texttt{torch.nn.functional.scaled\_dot\_product\_attention} with the FlashAttention 2 backend for all exact attention calls.
Runtime in \cref{fig:long-form} (right) is the sum of exact prefill attention, query attention during decoding, and compression time during decoding over all $32$ layers. Because MATH-500 prompts are short, we compute the context keys and values with exact attention for all methods and apply the attention approximation methods only during decoding.
We use one warm-up run and one timed run for all long-form decoding experiments.
\begin{table}[h!]
    \centering
    \caption{\tbf{Configurations for the attention approximation methods of \cref{fig:long-form}.}}
    \begin{tabular}{cc}
        \toprule
        {\bf Attention Algorithm} & {\bf Configuration}
        \\\midrule
        \multirow{1}{*}{\bf StreamingLLM} &
        \begin{tabular}[c]{@{}c@{}}
        \texttt{n\_sink=4},
        \texttt{compression\_interval=1}
        \end{tabular}
        \\[1mm]
        \multirow{1}{*}{\bf SnapKV} &
        \begin{tabular}[c]{@{}c@{}}
        \texttt{window\_size=64},
        \texttt{kernel\_size=5}, \\
        \texttt{compression\_interval=128}
        \end{tabular}
        \\[1mm]
        \multirow{1}{*}{\bf ExpectedAttention} &
        \begin{tabular}[c]{@{}c@{}}
        \texttt{n\_future\_positions=512},
        \texttt{n\_sink=4}, \\
        \texttt{use\_covariance=True},
        \texttt{use\_vnorm=True}, \\
        \texttt{epsilon=0.01},
        \texttt{compression\_interval=128}
        \end{tabular}
        \\[1mm]
        \multirow{1}{*}{\bf KeyDiff} & \texttt{compression\_interval=128}
        \\[1mm]
        \multirow{1}{*}{\bf Knorm} & \texttt{compression\_interval=128}
        \\[1mm]
        \multirow{1}{*}{\bf \ours (Ours)} &
        \begin{tabular}[c]{@{}c@{}}
        \texttt{n\_sink=32},
        \texttt{n\_window=32}
        \end{tabular}
        \\
        \bottomrule
    \end{tabular}
    \label{tab:long-form-configs}
\end{table}

\subsection{Long-form decoding with BalanceKV Express}
\label{app:balancekv-express}

\Cref{fig:long-form-balancekv} repeats the memory- and compute-constrained long-form decoding experiments of \cref{fig:long-form} with \ours replaced by BalanceKV Express, the result of applying \express to the unmasked halving algorithm of BalanceKV~\citep{han2025streaming}.
All other settings match those of \cref{app:long_form_configs}.

\begin{figure}[h!]
    \centering
    \includegraphics[width=0.49\textwidth]{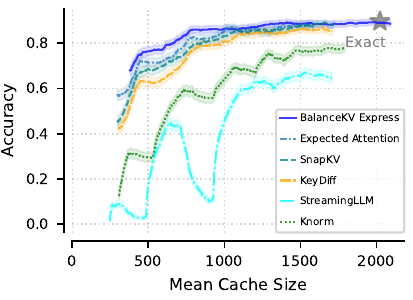}\hfill
    \includegraphics[width=0.49\textwidth]{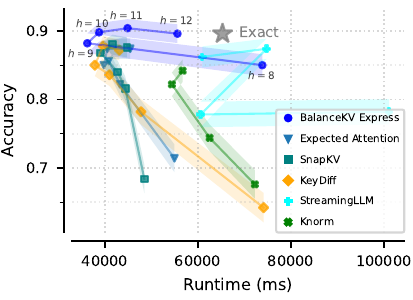}
    \caption{\textbf{Memory- and compute-efficient long-form decoding with BalanceKV Express.} On competition-level MATH-500  problems that benefit from step-by-step reasoning, applying Express to the unmasked halving algorithm of BalanceKV~\citep{han2025streaming} improves both the cache size-accuracy trade-off (left) and time-accuracy trade-off (right) of long-form decoding. Error bands display $\pm 1$ standard error across all 500 benchmark questions. See \cref{app:balancekv-express} for more details.
    }
    \label{fig:long-form-balancekv}
\end{figure}

\clearpage\newpage
\section*{NeurIPS Paper Checklist}

\begin{enumerate}

\item {\bf Claims}
    \item[] Question: Do the main claims made in the abstract and introduction accurately reflect the paper's contributions and scope?
    \item[] Answer: \answerYes{} %
    \item[] Justification: 
    \item[] Guidelines:
    \begin{itemize}
        \item The answer \answerNA{} means that the abstract and introduction do not include the claims made in the paper.
        \item The abstract and/or introduction should clearly state the claims made, including the contributions made in the paper and important assumptions and limitations. A \answerNo{} or \answerNA{} answer to this question will not be perceived well by the reviewers. 
        \item The claims made should match theoretical and experimental results, and reflect how much the results can be expected to generalize to other settings. 
        \item It is fine to include aspirational goals as motivation as long as it is clear that these goals are not attained by the paper. 
    \end{itemize}

\item {\bf Limitations}
    \item[] Question: Does the paper discuss the limitations of the work performed by the authors?
    \item[] Answer: \answerYes{} %
    \item[] Justification: %
    \item[] Guidelines:
    \begin{itemize}
        \item The answer \answerNA{} means that the paper has no limitation while the answer \answerNo{} means that the paper has limitations, but those are not discussed in the paper. 
        \item The authors are encouraged to create a separate ``Limitations'' section in their paper.
        \item The paper should point out any strong assumptions and how robust the results are to violations of these assumptions (e.g., independence assumptions, noiseless settings, model well-specification, asymptotic approximations only holding locally). The authors should reflect on how these assumptions might be violated in practice and what the implications would be.
        \item The authors should reflect on the scope of the claims made, e.g., if the approach was only tested on a few datasets or with a few runs. In general, empirical results often depend on implicit assumptions, which should be articulated.
        \item The authors should reflect on the factors that influence the performance of the approach. For example, a facial recognition algorithm may perform poorly when image resolution is low or images are taken in low lighting. Or a speech-to-text system might not be used reliably to provide closed captions for online lectures because it fails to handle technical jargon.
        \item The authors should discuss the computational efficiency of the proposed algorithms and how they scale with dataset size.
        \item If applicable, the authors should discuss possible limitations of their approach to address problems of privacy and fairness.
        \item While the authors might fear that complete honesty about limitations might be used by reviewers as grounds for rejection, a worse outcome might be that reviewers discover limitations that aren't acknowledged in the paper. The authors should use their best judgment and recognize that individual actions in favor of transparency play an important role in developing norms that preserve the integrity of the community. Reviewers will be specifically instructed to not penalize honesty concerning limitations.
    \end{itemize}

\item {\bf Theory assumptions and proofs}
    \item[] Question: For each theoretical result, does the paper provide the full set of assumptions and a complete (and correct) proof?
    \item[] Answer: \answerYes{} %
    \item[] Justification: 
    \item[] Guidelines:
    \begin{itemize}
        \item The answer \answerNA{} means that the paper does not include theoretical results. 
        \item All the theorems, formulas, and proofs in the paper should be numbered and cross-referenced.
        \item All assumptions should be clearly stated or referenced in the statement of any theorems.
        \item The proofs can either appear in the main paper or the supplemental material, but if they appear in the supplemental material, the authors are encouraged to provide a short proof sketch to provide intuition. 
        \item Inversely, any informal proof provided in the core of the paper should be complemented by formal proofs provided in appendix or supplemental material.
        \item Theorems and Lemmas that the proof relies upon should be properly referenced. 
    \end{itemize}

    \item {\bf Experimental result reproducibility}
    \item[] Question: Does the paper fully disclose all the information needed to reproduce the main experimental results of the paper to the extent that it affects the main claims and/or conclusions of the paper (regardless of whether the code and data are provided or not)?
    \item[] Answer: \answerYes{} %
    \item[] Justification: 
    \item[] Guidelines:
    \begin{itemize}
        \item The answer \answerNA{} means that the paper does not include experiments.
        \item If the paper includes experiments, a \answerNo{} answer to this question will not be perceived well by the reviewers: Making the paper reproducible is important, regardless of whether the code and data are provided or not.
        \item If the contribution is a dataset and\slash or model, the authors should describe the steps taken to make their results reproducible or verifiable. 
        \item Depending on the contribution, reproducibility can be accomplished in various ways. For example, if the contribution is a novel architecture, describing the architecture fully might suffice, or if the contribution is a specific model and empirical evaluation, it may be necessary to either make it possible for others to replicate the model with the same dataset, or provide access to the model. In general. releasing code and data is often one good way to accomplish this, but reproducibility can also be provided via detailed instructions for how to replicate the results, access to a hosted model (e.g., in the case of a large language model), releasing of a model checkpoint, or other means that are appropriate to the research performed.
        \item While NeurIPS does not require releasing code, the conference does require all submissions to provide some reasonable avenue for reproducibility, which may depend on the nature of the contribution. For example
        \begin{enumerate}
            \item If the contribution is primarily a new algorithm, the paper should make it clear how to reproduce that algorithm.
            \item If the contribution is primarily a new model architecture, the paper should describe the architecture clearly and fully.
            \item If the contribution is a new model (e.g., a large language model), then there should either be a way to access this model for reproducing the results or a way to reproduce the model (e.g., with an open-source dataset or instructions for how to construct the dataset).
            \item We recognize that reproducibility may be tricky in some cases, in which case authors are welcome to describe the particular way they provide for reproducibility. In the case of closed-source models, it may be that access to the model is limited in some way (e.g., to registered users), but it should be possible for other researchers to have some path to reproducing or verifying the results.
        \end{enumerate}
    \end{itemize}

\item {\bf Open access to data and code}
    \item[] Question: Does the paper provide open access to the data and code, with sufficient instructions to faithfully reproduce the main experimental results, as described in supplemental material?
    \item[] Answer: \answerYes{} %
    \item[] Justification:
    \item[] Guidelines:
    \begin{itemize}
        \item The answer \answerNA{} means that paper does not include experiments requiring code.
        \item Please see the NeurIPS code and data submission guidelines (\url{https://neurips.cc/public/guides/CodeSubmissionPolicy}) for more details.
        \item While we encourage the release of code and data, we understand that this might not be possible, so \answerNo{} is an acceptable answer. Papers cannot be rejected simply for not including code, unless this is central to the contribution (e.g., for a new open-source benchmark).
        \item The instructions should contain the exact command and environment needed to run to reproduce the results. See the NeurIPS code and data submission guidelines (\url{https://neurips.cc/public/guides/CodeSubmissionPolicy}) for more details.
        \item The authors should provide instructions on data access and preparation, including how to access the raw data, preprocessed data, intermediate data, and generated data, etc.
        \item The authors should provide scripts to reproduce all experimental results for the new proposed method and baselines. If only a subset of experiments are reproducible, they should state which ones are omitted from the script and why.
        \item At submission time, to preserve anonymity, the authors should release anonymized versions (if applicable).
        \item Providing as much information as possible in supplemental material (appended to the paper) is recommended, but including URLs to data and code is permitted.
    \end{itemize}

\item {\bf Experimental setting/details}
    \item[] Question: Does the paper specify all the training and test details (e.g., data splits, hyperparameters, how they were chosen, type of optimizer) necessary to understand the results?
    \item[] Answer: \answerYes{} %
    \item[] Justification: 
    \item[] Guidelines:
    \begin{itemize}
        \item The answer \answerNA{} means that the paper does not include experiments.
        \item The experimental setting should be presented in the core of the paper to a level of detail that is necessary to appreciate the results and make sense of them.
        \item The full details can be provided either with the code, in appendix, or as supplemental material.
    \end{itemize}

\item {\bf Experiment statistical significance}
    \item[] Question: Does the paper report error bars suitably and correctly defined or other appropriate information about the statistical significance of the experiments?
    \item[] Answer: \answerYes{} %
    \item[] Justification: 
    \item[] Guidelines:
    \begin{itemize}
        \item The answer \answerNA{} means that the paper does not include experiments.
        \item The authors should answer \answerYes{} if the results are accompanied by error bars, confidence intervals, or statistical significance tests, at least for the experiments that support the main claims of the paper.
        \item The factors of variability that the error bars are capturing should be clearly stated (for example, train/test split, initialization, random drawing of some parameter, or overall run with given experimental conditions).
        \item The method for calculating the error bars should be explained (closed form formula, call to a library function, bootstrap, etc.)
        \item The assumptions made should be given (e.g., Normally distributed errors).
        \item It should be clear whether the error bar is the standard deviation or the standard error of the mean.
        \item It is OK to report 1-sigma error bars, but one should state it. The authors should preferably report a 2-sigma error bar than state that they have a 96\% CI, if the hypothesis of Normality of errors is not verified.
        \item For asymmetric distributions, the authors should be careful not to show in tables or figures symmetric error bars that would yield results that are out of range (e.g., negative error rates).
        \item If error bars are reported in tables or plots, the authors should explain in the text how they were calculated and reference the corresponding figures or tables in the text.
    \end{itemize}

\item {\bf Experiments compute resources}
    \item[] Question: For each experiment, does the paper provide sufficient information on the computer resources (type of compute workers, memory, time of execution) needed to reproduce the experiments?
    \item[] Answer: \answerYes{} %
    \item[] Justification: 
    \item[] Guidelines:
    \begin{itemize}
        \item The answer \answerNA{} means that the paper does not include experiments.
        \item The paper should indicate the type of compute workers CPU or GPU, internal cluster, or cloud provider, including relevant memory and storage.
        \item The paper should provide the amount of compute required for each of the individual experimental runs as well as estimate the total compute. 
        \item The paper should disclose whether the full research project required more compute than the experiments reported in the paper (e.g., preliminary or failed experiments that didn't make it into the paper). 
    \end{itemize}
    
\item {\bf Code of ethics}
    \item[] Question: Does the research conducted in the paper conform, in every respect, with the NeurIPS Code of Ethics \url{https://neurips.cc/public/EthicsGuidelines}?
    \item[] Answer: \answerYes{}%
    \item[] Justification: 
    \item[] Guidelines:
    \begin{itemize}
        \item The answer \answerNA{} means that the authors have not reviewed the NeurIPS Code of Ethics.
        \item If the authors answer \answerNo, they should explain the special circumstances that require a deviation from the Code of Ethics.
        \item The authors should make sure to preserve anonymity (e.g., if there is a special consideration due to laws or regulations in their jurisdiction).
    \end{itemize}

\item {\bf Broader impacts}
    \item[] Question: Does the paper discuss both potential positive societal impacts and negative societal impacts of the work performed?
    \item[] Answer: \answerYes{} %
    \item[] Justification: %
    \item[] Guidelines:
    \begin{itemize}
        \item The answer \answerNA{} means that there is no societal impact of the work performed.
        \item If the authors answer \answerNA{} or \answerNo, they should explain why their work has no societal impact or why the paper does not address societal impact.
        \item Examples of negative societal impacts include potential malicious or unintended uses (e.g., disinformation, generating fake profiles, surveillance), fairness considerations (e.g., deployment of technologies that could make decisions that unfairly impact specific groups), privacy considerations, and security considerations.
        \item The conference expects that many papers will be foundational research and not tied to particular applications, let alone deployments. However, if there is a direct path to any negative applications, the authors should point it out. For example, it is legitimate to point out that an improvement in the quality of generative models could be used to generate Deepfakes for disinformation. On the other hand, it is not needed to point out that a generic algorithm for optimizing neural networks could enable people to train models that generate Deepfakes faster.
        \item The authors should consider possible harms that could arise when the technology is being used as intended and functioning correctly, harms that could arise when the technology is being used as intended but gives incorrect results, and harms following from (intentional or unintentional) misuse of the technology.
        \item If there are negative societal impacts, the authors could also discuss possible mitigation strategies (e.g., gated release of models, providing defenses in addition to attacks, mechanisms for monitoring misuse, mechanisms to monitor how a system learns from feedback over time, improving the efficiency and accessibility of ML).
    \end{itemize}
    
\item {\bf Safeguards}
    \item[] Question: Does the paper describe safeguards that have been put in place for responsible release of data or models that have a high risk for misuse (e.g., pre-trained language models, image generators, or scraped datasets)?
    \item[] Answer: \answerNA{} %
    \item[] Justification: 
    \item[] Guidelines:
    \begin{itemize}
        \item The answer \answerNA{} means that the paper poses no such risks.
        \item Released models that have a high risk for misuse or dual-use should be released with necessary safeguards to allow for controlled use of the model, for example by requiring that users adhere to usage guidelines or restrictions to access the model or implementing safety filters. 
        \item Datasets that have been scraped from the Internet could pose safety risks. The authors should describe how they avoided releasing unsafe images.
        \item We recognize that providing effective safeguards is challenging, and many papers do not require this, but we encourage authors to take this into account and make a best faith effort.
    \end{itemize}

\item {\bf Licenses for existing assets}
    \item[] Question: Are the creators or original owners of assets (e.g., code, data, models), used in the paper, properly credited and are the license and terms of use explicitly mentioned and properly respected?
    \item[] Answer: \answerYes{} %
    \item[] Justification:
    \item[] Guidelines:
    \begin{itemize}
        \item The answer \answerNA{} means that the paper does not use existing assets.
        \item The authors should cite the original paper that produced the code package or dataset.
        \item The authors should state which version of the asset is used and, if possible, include a URL.
        \item The name of the license (e.g., CC-BY 4.0) should be included for each asset.
        \item For scraped data from a particular source (e.g., website), the copyright and terms of service of that source should be provided.
        \item If assets are released, the license, copyright information, and terms of use in the package should be provided. For popular datasets, \url{paperswithcode.com/datasets} has curated licenses for some datasets. Their licensing guide can help determine the license of a dataset.
        \item For existing datasets that are re-packaged, both the original license and the license of the derived asset (if it has changed) should be provided.
        \item If this information is not available online, the authors are encouraged to reach out to the asset's creators.
    \end{itemize}

\item {\bf New assets}
    \item[] Question: Are new assets introduced in the paper well documented and is the documentation provided alongside the assets?
    \item[] Answer: \answerYes{} %
    \item[] Justification: 
    \item[] Guidelines:
    \begin{itemize}
        \item The answer \answerNA{} means that the paper does not release new assets.
        \item Researchers should communicate the details of the dataset\slash code\slash model as part of their submissions via structured templates. This includes details about training, license, limitations, etc. 
        \item The paper should discuss whether and how consent was obtained from people whose asset is used.
        \item At submission time, remember to anonymize your assets (if applicable). You can either create an anonymized URL or include an anonymized zip file.
    \end{itemize}

\item {\bf Crowdsourcing and research with human subjects}
    \item[] Question: For crowdsourcing experiments and research with human subjects, does the paper include the full text of instructions given to participants and screenshots, if applicable, as well as details about compensation (if any)? 
    \item[] Answer: \answerNA{} %
    \item[] Justification: 
    \item[] Guidelines:
    \begin{itemize}
        \item The answer \answerNA{} means that the paper does not involve crowdsourcing nor research with human subjects.
        \item Including this information in the supplemental material is fine, but if the main contribution of the paper involves human subjects, then as much detail as possible should be included in the main paper. 
        \item According to the NeurIPS Code of Ethics, workers involved in data collection, curation, or other labor should be paid at least the minimum wage in the country of the data collector. 
    \end{itemize}

\item {\bf Institutional review board (IRB) approvals or equivalent for research with human subjects}
    \item[] Question: Does the paper describe potential risks incurred by study participants, whether such risks were disclosed to the subjects, and whether Institutional Review Board (IRB) approvals (or an equivalent approval/review based on the requirements of your country or institution) were obtained?
    \item[] Answer: \answerNA{} %
    \item[] Justification: 
    \item[] Guidelines:
    \begin{itemize}
        \item The answer \answerNA{} means that the paper does not involve crowdsourcing nor research with human subjects.
        \item Depending on the country in which research is conducted, IRB approval (or equivalent) may be required for any human subjects research. If you obtained IRB approval, you should clearly state this in the paper. 
        \item We recognize that the procedures for this may vary significantly between institutions and locations, and we expect authors to adhere to the NeurIPS Code of Ethics and the guidelines for their institution. 
        \item For initial submissions, do not include any information that would break anonymity (if applicable), such as the institution conducting the review.
    \end{itemize}

\item {\bf Declaration of LLM usage}
    \item[] Question: Does the paper describe the usage of LLMs if it is an important, original, or non-standard component of the core methods in this research? Note that if the LLM is used only for writing, editing, or formatting purposes and does \emph{not} impact the core methodology, scientific rigor, or originality of the research, declaration is not required.
    \item[] Answer: \answerNA{} %
    \item[] Justification: 
    \item[] Guidelines:
    \begin{itemize}
        \item The answer \answerNA{} means that the core method development in this research does not involve LLMs as any important, original, or non-standard components.
        \item Please refer to our LLM policy in the NeurIPS handbook for what should or should not be described.
    \end{itemize}

\end{enumerate}

\end{document}